\documentclass{article}

\usepackage[numbers, compress]{natbib}
\usepackage[utf8]{inputenc} 
\usepackage[T1]{fontenc}    
\usepackage{hyperref}       
\usepackage{url}            
\usepackage{booktabs}       
\usepackage{amsfonts}       
\usepackage{nicefrac}       
\usepackage{microtype}      
\usepackage{fullpage}

\usepackage{amsmath, amssymb, amsthm}
\usepackage{color, soul}
\usepackage{algorithm, algorithmic, multicol}
\usepackage[titletoc,title]{appendix}
\usepackage{graphicx}

\usepackage{bbm}

\newcommand\reals{\mathbb{R}} 
\newcommand\naturals{\mathbb{N}} 
\newcommand\integers{\mathbb{Z}} 
\newcommand\E{\mathbb{E}} 
\newcommand\prob{\mathbb{P}} 
\DeclareMathOperator*{\argmax}{argmax} 
\newcommand\ind{\mathbbm{1}}  
\newcommand\cost{\textbf{C}} 
\newcommand\normcost{\textbf{D}} 
\newcommand\costseor{\mathcal{C}^{eor}} 
\newcommand\reducedcostseor{\mathcal{C}^{eor}_1} 
\newcommand\weakspace{\mathcal{H}} 
\newcommand\unifv{\textbf{u}} 
\newcommand\unifm{\textbf{U}} 
\newcommand\samplex{\textbf{x}} 
\newcommand\stdv{\textbf{e}} 
\newcommand\cumv{\textbf{s}} 
\newcommand\weightv{\textbf{w}} 
\newcommand\predy{\hat y} 
\newcommand\potential{\phi} 
\newcommand\probv{\textbf{p}} 
\newcommand\zerov{\textbf{0}} 

\newtheorem{theorem}{Theorem}
\newtheorem{definition}[theorem]{Definition}
\newtheorem{lemma}[theorem]{Lemma}
\newtheorem{corollary}[theorem]{Corollary}
\newtheorem*{remark}{Remark}

\title{Online Multiclass Boosting}

\author{
  Young Hun Jung\\
  \and
  Jack Goetz\\
  Department of Statistics\\
  University of Michigan\\
  Ann Arbor, MI 48109 \\
  \texttt{\{yhjung, jrgoetz, tewaria\}@umich.edu} \\
  \and
  Ambuj Tewari
}

\begin{document}

\maketitle

\begin{abstract}
Recent work has extended the theoretical analysis of boosting algorithms to multiclass problems and to online settings. However, the multiclass extension is in the batch setting and the online extensions only consider binary classification. We fill this gap in the literature by defining, and justifying, a weak learning condition for online multiclass boosting. This condition leads to an optimal boosting algorithm that requires the minimal number of weak learners to achieve a certain accuracy. Additionally, we propose an adaptive algorithm which is near optimal and enjoys an excellent performance on real data due to its adaptive property. 
\end{abstract} 

\section{Introduction}
\textit{Boosting} methods are a ensemble learning methods that aggregate several (not necessarily) weak learners to build a stronger learner. When used to aggregate reasonably strong learners, boosting has been shown to produce results competitive with other state-of-the-art methods (e.g., \citet{korytkowski2016fast}, \citet{zhang2014boosted}). Until recently theoretical development in this area has been focused on batch binary settings where the learner can observe the entire training set at once, and the labels are restricted to be binary (cf. \citet{schapire2012boosting}). In the past few years, progress has been made to extend the theory and algorithms to more general settings.

Dealing with \textit{multiclass classification} turned out to be more subtle than initially expected. \citet{mukherjee2013theory} unify several different proposals made earlier in the literature and provide a general framework for multiclass boosting. They state their weak learning conditions in terms of \textit{cost matrices} that have to satisfy certain restrictions: for example, labeling with the ground truth should have less cost than labeling with some other labels. A weak learning condition, just like the binary condition, states that the performance of a learner, now judged using a cost matrix, should be better than a random guessing baseline. One particular condition they call the \textit{edge-over-random} condition, proves to be sufficient for boostability. The edge-over-random condition will also figure prominently in this paper. They also consider a necessary and sufficient condition for boostability but it turns out to be computationally intractable to be used in practice.

A recent trend in modern machine learning is to train learners in an \textit{online setting} where the instances come sequentially and the learner has to make predictions instantly. \citet{oza2005online} initially proposed an online boosting algorithm that has accuracy comparable with the batch version, but it took several years to design an algorithm with theoretical justification (\citet{chen2012online}). \citet{beygelzimer2015optimal} achieved a breakthrough by proposing an optimal algorithm in online binary settings and an adaptive algorithm that works quite well in practice. These theories in online binary boosting have led to several extensions. For example, \citet{chen2014boosting} combine one vs all method with binary boosting algorithms to tackle online multiclass problems with bandit feedback, and \citet{hu2017gradient} build a theory of boosting in regression setting.

In this paper, we combine the insights and techniques of \citet{mukherjee2013theory} and \citet{beygelzimer2015optimal} to provide a framework for online multiclass boosting. The cost matrix framework from the former work is adopted to propose an online weak learning condition that defines how well a learner can perform over a random guess (Definition \ref{def:onlineWLC}). We show this condition is naturally derived from its batch setting counterpart. From this weak learning condition, a boosting algorithm  (Algorithm \ref{alg:onlineMBBM}) is proposed which is theoretically optimal in that it requires the minimal number of learners and sample complexity to attain a specified level of accuracy. We also develop an adaptive algorithm (Algorithm \ref{alg:AdaboostOLM}) which allows learners to have variable strengths. This algorithm is theoretically less efficient than the optimal one, but the experimental results show that it is quite comparable and sometimes even better due to its adaptive property. Both algorithms not only possess theoretical proofs of mistake bounds, but also demonstrate superior performance over preexisting methods. 

\section{Preliminaries}
We first describe the basic setup for online boosting. While in the batch setting, an additional weak learner is trained at every iteration, in the online setting, the algorithm starts with a fixed count of $N$ \textit{weak learners} and a \textit{booster} which manages the weak learners. There are $k$ possible labels $[k] := \{1, \cdots, k\}$ and $k$ is known to the learners. At each iteration $t=1, \cdots, T$, an \textit{adversary} picks a labeled example $(\samplex_{t}, y_{t}) \in \mathcal{X} \times [k]$, where $\mathcal{X}$ is some domain, and reveals $\samplex_{t}$ to the booster. Once the booster observes the unlabeled data $\samplex_{t}$, it gathers the weak learners' predictions and makes a final prediction. Throughout this paper, index $i$ takes values from $1$ to $N$; $t$ from 1 to $T$; and $l$ from 1 to $k$. 

We utilize the \textit{cost matrix framework}, first proposed by \citet{mukherjee2013theory}, to develop multiclass boosting algorithms. This is a key ingredient in the multiclass extension as it enables different penalization for each pair of correct label and prediction, and we further develop this framework to suit the online setting. The booster sequentially computes \textit{cost matrices} $\{\cost^{i}_{t} \in \reals^{k \times k}~|~i = 1, \cdots, N\}$, sends $(\samplex_{t}, \cost^{i}_{t})$ to the $i^{th}$ weak learner $WL^{i}$, and gets its prediction $l^{i}_{t} \in [k]$. Here the cost matrix $\cost^{i}_{t}$ plays a role of loss function in that $WL^{i}$ tries to minimize the cumulative cost $\sum_{t} \cost^{i}_{t}[y_{t}, l^{i}_{t}]$. As the booster wants each learner to predict the correct label, it wants to set the diagonal entries of $\cost^{i}_{t}$ to be minimal among its row. At this stage, the true label $y_{t}$ is not revealed yet, but the previous weak learners' predictions can affect the computation of the cost matrix for the next learner. Given a matrix $\cost$, the $(i, j)^{th}$ entry will be denoted by $\cost[i, j]$, and $i^{th}$ row vector by $\cost[i]$. 

Once all the learners make predictions, the booster makes the final prediction $\predy_{t}$ by majority votes. The booster can either take simple majority votes or weighted ones. In fact for the adaptive algorithm, we will allow weighted votes so that the booster can assign more weights on well-performing learners. The weight for $WL^{i}$ at iteration $t$ will be denoted by $\alpha^{i}_{t}$. After observing the booster's final decision, the adversary reveals the true label $y_{t}$, and the booster suffers 0-1 loss $\ind(\predy_{t} \neq y_{t})$. The booster also shares the true label to the weak learners so that they can train on this data point. 

Two main issues have to be resolved to design a good boosting algorithm. First, we need to design the booster's strategy for producing cost matrices. Second, we need to quantify weak learner's ability to reduce the cumulative cost $\sum_{t=1}^{T}\cost^{i}_{t}[y_{t}, l^{i}_{t}]$. The first issue will be resolved by introducing potential functions, which will be thoroughly discussed in Section \ref{section:generalOptimal}. For the second issue, we introduce our online weak learning condition, a generalization of the weak learning assumption in \citet{beygelzimer2015optimal}, stating that for any adaptively given sequence of cost matrices, weak learners can produce predictions whose cumulative cost is less than that incurred by random guessing. The online weak learning condition will be discussed in the following section. For the analysis of the adaptive algorithm, we use empirical edges instead of the online weak learning condition.

\subsection{Online weak learning condition}
\label{section:onlineWLC}

In this section, we propose an online weak learning condition that states the weak learners are better than a random guess. We first define a baseline condition that is better than a random guess. Let $\Delta [k]$ denote a family of distributions over $[k]$ and $\unifv^{l}_{\gamma} \in \Delta [k]$ be a uniform distribution that puts $\gamma$ more weight on the label $l$. For example, $\unifv^{1}_{\gamma} = (\frac{1-\gamma}{k}+\gamma, \frac{1-\gamma}{k}, \cdots, \frac{1-\gamma}{k})$. For a given sequence of examples $\{(\samplex_{t}, y_{t}) ~|~ t = 1, \cdots, T\}$, $\unifm_{\gamma}\in \reals^{T \times k}$ consists of rows $\unifv^{y_{t}}_{\gamma}$. Then we restrict the booster's choice of cost matrices to
\begin{equation*}
\reducedcostseor:= \{\cost \in \reals^{k\times k} ~|~ \forall l, r \in [k], ~\cost[l, l] = 0, \cost[l, r] \geq 0,  \text{ and } ||\cost[l]||_{1} = 1 \}.
\end{equation*}
Note that diagonal entries are minimal among the row, and $\reducedcostseor$ also has a normalization constraint. A broader choice of cost matrices is allowed if one can assign importance weights on observations, which is possible for various learners. Even if the learner does not take the importance weight as an input, we can achieve a similar effect by sending to the learner an instance with probability that is proportional to its weight. Interested readers can refer \citet[Lemma 1]{beygelzimer2015optimal}. From now on, we will assume that our weak learners can take weight $w_{t}$ as an input. 

We are ready to present our online weak learning condition. This condition is in fact naturally derived from the batch setting counterpart that is well studied by \citet{mukherjee2013theory}. The link is thoroughly discussed in Appendix \ref{appendix:linkWLC}. For the scaling issue, we assume the weights $w_{t}$ lie in $[0, 1]$. 

\begin{definition}{\bf(Online multiclass weak learning condition)}
\label{def:onlineWLC}
For parameters $\gamma, \delta \in (0, 1)$, and $S > 0$, a pair of online learner and an adversary is said to satisfy online weak learning condition with parameters $\delta, \gamma, \text{and }S$ if for any sample length $T$, any adaptive sequence of labeled examples, and for any adaptively chosen series of pairs of weight and cost matrix $\{(w_{t}, \cost_{t}) \in [0, 1] \times \reducedcostseor~|~t = 1, \cdots, T\}$, the learner can generate predictions $\predy_{t}$ such that with probability at least $1-\delta$, 
\begin{equation}
\label{eq:onlineWLC}
\sum_{t=1}^{T}w_{t}\cost_{t}[y_{t}, \predy_{t}] \leq \cost \bullet \unifm_{\gamma}'+S = \frac{1-\gamma}{k}||\weightv||_{1} + S,
\end{equation}
where $\cost \in \reals^{T \times k}$ consists of rows of $w_{t}\cost_{t}[y_{t}]$ and $\textbf{A} \bullet \textbf{B}'$ denotes the Frobenius inner product $\text{Tr}(\textbf{A} \textbf{B}')$. $\weightv = (w_{1}, \cdots, w_{T})$ and the last equality holds due to the normalized condition on $\reducedcostseor$. $\gamma$ is called an edge, and $S$ an excess loss. 
\end{definition}

\begin{remark}
	Notice that this condition is imposed on a pair of learner and adversary instead of solely on a learner. This is because no learner can satisfy this condition if the adversary draws samples in a completely adaptive manner. The probabilistic statement is necessary because many online algorithms' predictions are not deterministic. The excess loss requirement is needed since an online learner cannot produce meaningful predictions before observing a sufficient number of examples. 
\end{remark}

\section{An optimal algorithm}
In this section, we describe the booster's optimal strategy for designing cost matrices. We first introduce a general theory without specifying the loss, and later investigate the asymptotic behavior of cumulative loss suffered by our algorithm under the specific 0-1 loss. We adopt the potential function framework from \citet{mukherjee2013theory} and extend it to the online setting. Potential functions help both in designing cost matrices and in proving the mistake bound of the algorithm. 
\subsection{A general online multiclass boost-by-majority (OnlineMBBM) algorithm}
\label{section:generalOptimal}
We will keep track of the weighted cumulative votes of the first $i$ weak learners for the sample $\samplex_{t}$ by $\cumv^{i}_{t}:= \sum_{j=1}^{i}\alpha^{j}_{t} \stdv_{l^{j}_{t}}$, where $\alpha^{i}_{t}$ is the weight of $WL^{i}$, $l^{i}_{t}$ is its prediction and $\stdv_j$ is the $j^{th}$ standard basis vector. For the optimal algorithm, we assume that $\alpha^{i}_{t}=1,~ \forall i, t$. In other words, the booster makes the final decision by simple majority votes. Given a cumulative vote $\cumv \in \reals^{k}$, suppose we have a loss function $L^{r}(\cumv)$ where $r$ denotes the correct label. We call a loss function \textit{proper}, if it is a decreasing function of $\cumv[r]$ and an increasing function of other coordinates (we alert the reader that ``proper loss'' has at least one other meaning in the literature). From now on, we will assume that our loss function is proper. A good example of proper loss is multiclass 0-1 loss:
\begin{equation}
	\label{eq:zeroOne}
	L^{r}(\cumv) := \ind(\max_{l \neq r}\cumv[l] \geq \cumv[r]).
\end{equation}
The purpose of the potential function $\potential^{r}_{i}(\cumv)$ is to estimate the booster's loss when there remain $i$ learners until the final decision and the current cumulative vote is $\cumv$. More precisely, we want potential functions to satisfy the following conditions:
\begin{align}
	\begin{split}
		\label{eq:potentialCond}
		\potential^{r}_{0}(\cumv) &= L^{r}(\cumv), \\
		\potential^{r}_{i+1}(\cumv) &= \E_{l \sim \unifv^{r}_{\gamma}}\potential^{r}_{i}(\cumv +\stdv_{l}).
	\end{split}
\end{align}
Readers should note that $\potential^{r}_{i}(\cumv)$ also inherits the proper property of the loss function, which can be shown by induction. The condition (\ref{eq:potentialCond}) can be loosened by replacing both equalities by inequalities ``$\geq$'', but in practice we usually use equalities.

Now we describe the booster's strategy for designing cost matrices. After observing $\samplex_{t}$, the booster sequentially sets a cost matrix $\cost^{i}_{t}$ for $WL^{i}$, gets the weak learner's prediction $l^{i}_{t}$ and uses this in the computation of the next cost matrix $\cost^{i+1}_{t}$. Ultimately, booster wants to set 
\begin{equation}
	\label{eq:costMatrix}
	\cost^{i}_{t}[r, l] = \potential^{r}_{N-i}(\cumv^{i-1}_{t} + \stdv_{l}).
\end{equation}
However, this cost matrix does not satisfy the condition of $\reducedcostseor$, and thus should be modified in order to utilize the weak learning condition. First to make the cost for the true label equal to $0$, we subtract $\cost^{i}_{t}[r, r]$ from every element of $\cost^{i}_{t}[r]$. Since the potential function is proper, our new cost matrix still has non-negative elements after the subtraction. We then normalize the row so that each row has $\ell_1$ norm equal to $1$. In other words, we get new normalized cost matrix
\begin{equation}
	\label{eq:normalizedCostMatrix}
	\normcost^{i}_{t}[r, l] = \frac{\potential^r_{N-i} (\cumv^{i-1}_t + \stdv_l) - \potential^r_{N-i} (\cumv^{i-1}_t + \stdv_r)}{\weightv^{i}[t]},
\end{equation}
where $\weightv^{i}[t] := \sum_{l=1}^{k}\potential^r_{N-i} (\cumv^{i-1}_t + \stdv_l) - \potential^r_{N-i} (\cumv^{i-1}_t + \stdv_r)$ plays the role of weight. It is still possible that a row vector $\cost^{i}_{t}[r]$ is a zero vector so that normalization is impossible. In this case, we just leave it as a zero vector. Our weak learning condition (\ref{eq:onlineWLC}) still works with cost matrices some of whose row vectors are zeros because however the learner predicts, it incurs no cost. 

\begin{algorithm}[t!]
	\begin{algorithmic}[1]
		\FOR{$t=1, \cdots, T$}
		\STATE Receive example $\samplex_t$
		\STATE Set $\cumv^0_t = \textbf{0} \in \reals^k$
		\FOR{$i = 1, \cdots, N$}
		\STATE Set the normalized cost matrix $\normcost^{i}_{t}$ according to (\ref{eq:normalizedCostMatrix}) and pass it to $WL^{i}$
		\STATE Get weak predictions $l^i_t = WL^i(\samplex_t)$ and update $\cumv^i_t = \cumv^{i-1}_t + \stdv_{l^i_t}$
		\ENDFOR
		\STATE Predict $\predy_t := \argmax_l \cumv^N_t[l]$ and receive true label $y_t$
		\FOR{$i = 1, \cdots, N$}
		\STATE Set $\weightv^i[t] = \sum_{l=1}^{k}[\potential^{y_{t}}_{N-i} (\cumv^{i-1}_t + \stdv_l) - \potential^{y_{t}}_{N-i} (\cumv^{i-1}_t + \stdv_{y_{t}})]$
		\STATE Pass training example with weight $(\samplex_t, y_t, \weightv^{i}[t])$ to $WL^i$ 
		\ENDFOR
		\ENDFOR
	\end{algorithmic}
	\caption{Online Multiclass Boost-by-Majority (OnlineMBBM)}
	\label{alg:onlineMBBM}
\end{algorithm}

After defining cost matrices, the rest of the algorithm is straightforward except we have to estimate $||\weightv^{i}||_{\infty}$ to normalize the weight. This is necessary because the weak learning condition assumes the weights lying in $[0, 1]$. We cannot compute the exact value of $||\weightv^{i}||_{\infty}$ until the last instance is revealed, which is fine as we need this value only in proving the mistake bound. The estimate $w^{i*}$ for $||\weightv^{i}||_{\infty}$ requires to specify the loss, and we postpone the technical parts to Appendix \ref{appendix:mistakeSpecific}. Interested readers may directly refer Lemma \ref{lemma:weightBound} before proceeding. Once the learners generate predictions after observing cost matrices, the final decision is made by simple majority votes. After the true label is revealed, the booster updates the weight and sends the labeled instance with weight to the weak learners. The pseudocode for the entire algorithm is depicted in Algorithm \ref{alg:onlineMBBM}. The algorithm is named after \citet[OnlineBBM]{beygelzimer2015optimal}, which is in fact OnlineMBBM with binary labels.

We present our first main result regarding the mistake bound of general OnlineMBBM. The proof appears in Appendix \ref{appendix:mistakeOptimal}  where the main idea is adopted from \citet[Lemma 3]{beygelzimer2015optimal}. 

\begin{theorem}{\bf{(Cumulative loss bound for OnlineMBBM)}}
	\label{thm:mistakeOptimal}
	Suppose weak learners and an adversary satisfy the online weak learning condition (\ref{eq:onlineWLC}) with parameters $\delta, \gamma, \text{and } S$. For any $T$ and $N$ satisfying $\delta \ll \frac{1}{N}$, and any adaptive sequence of labeled examples generated by the adversary, the final loss suffered by OnlineMBBM satisfies the following inequality with probability $1 - N\delta$:
	\begin{equation}
		\label{eq:mistakeGeneralOptimal}
		\sum_{t=1}^{T}L^{y_{t}}(\cumv^{N}_{t}) \leq \potential^{1}_{N}(\textbf{0}) T + S \sum_{i=1}^{N}w^{i*}.
	\end{equation}
\end{theorem}

Here $\potential^{1}_{N}(\textbf{0})$ plays a role of asymptotic error rate and the second term determines the sample complexity. We will investigate the behavior of those terms under the 0-1 loss in the following section.

\subsection{Mistake bound under 0-1 loss and its optimality}
From now on, we will specify the loss to be multiclass 0-1 loss defined in (\ref{eq:zeroOne}), which might be the most relevant measure in multiclass problems. To present a specific mistake bound, two terms in the RHS of (\ref{eq:mistakeGeneralOptimal}) should be bounded. This requires an approximation of potentials, which is technical and postponed to Appendix \ref{appendix:mistakeSpecific}. Lemma \ref{lemma:asymptoticError} and \ref{lemma:weightBound} provide the bounds for those terms. We also mention another bound for the weight in the remark after Lemma \ref{lemma:weightBound} so that one can use whichever tighter. Combining the above lemmas with Theorem \ref{thm:mistakeOptimal} gives the following corollary. The additional constraint on $\gamma$ comes from Lemma \ref{lemma:weightBound}. 

\begin{corollary}
	\label{corollary:mistakeOptimal}
	{\bf (0-1 loss bound of OnlineMBBM)}
	Suppose weak learners and an adversary satisfy the online weak learning condition (\ref{eq:onlineWLC}) with parameters $\delta, \gamma, \text{and } S$, where $\gamma < \frac{1}{2}$. For any $T$ and $N$ satisfying $\delta \ll \frac{1}{N}$ and any adaptive sequence of labeled examples generated by the adversary, OnlineMBBM can generate predictions $\predy_{t}$ that satisfy the following inequality with probability $1 - N\delta$:
	\begin{equation}
		\label{eq:mistakeOnlineMBBM}
		\sum_{t=1}^{T} \ind(y_{t} \neq \predy_{t}) \leq (k-1) e^{-\frac{\gamma^{2}N}{2}} T + \tilde O (k^{5/2}\sqrt N S).
	\end{equation}
	Therefore in order to achieve error rate $\epsilon$, it suffices to use $N = \Theta (\frac{1}{\gamma^{2}}\ln \frac{k}{\epsilon})$ weak learners, which gives an excess loss bound of $\tilde \Theta (\frac{k^{5/2}}{\gamma}S)$. 
\end{corollary}

\begin{remark}
Note that the above excess loss bound gives a sample complexity bound of $\tilde \Theta (\frac{k^{5/2}}{\epsilon\gamma}S)$. If we use alternative weight bound to get $kNS$ as an upper bound for the second term in (\ref{eq:mistakeGeneralOptimal}), we end up having $\tilde O(kNS)$. This will give an excess loss bound of $\tilde \Theta(\frac{k}{\gamma^{2}}S)$. 
\end{remark}

We now provide lower bounds on the number of learners and sample complexity for arbitrary online boosting algorithms to evaluate the optimality of OnlineMBBM under 0-1 loss. In particular, we construct weak learners that satisfy the online weak learning condition (\ref{eq:onlineWLC}) and have almost matching asymptotic error rate and excess loss compared to those of OnlineMBBM as in (\ref{eq:mistakeOnlineMBBM}). Indeed we can prove that the number of learners and sample complexity of OnlineMBBM is optimal up to logarithmic factors, ignoring the influence of the number of classes $k$. Our bounds are possibly suboptimal up to polynomial factors in $k$, and the problem to fill the gap remains open. The detailed proof and a discussion of the gap can be found in Appendix \ref{appendix:optimality}. Our lower bound is a multiclass version of \citet[Theorem 3]{beygelzimer2015optimal}. 

\begin{theorem}{\bf{(Lower bounds for $N$ and $T$)}}
	\label{thm:optimality}
	For any $\gamma \in (0, \frac{1}{4})$,  $\delta, \epsilon \in (0, 1)$, and $S \geq \frac{k\ln(\frac{1}{\delta})}{\gamma}$, there exists an adversary with a family of learners satisfying the online weak learning condition (\ref{eq:onlineWLC}) with parameters $\delta, \gamma$, and $S$, such that to achieve asymptotic error rate $\epsilon$, an online boosting algorithm requires at least $\Omega(\frac{1}{k^{2}\gamma^{2}}\ln \frac{1}{\epsilon})$ learners and a sample complexity of $\Omega (\frac{k}{\epsilon\gamma}S)$.
\end{theorem}

\section{An adaptive algorithm}
The online weak learning condition imposes minimal assumptions on the asymptotic accuracy of learners, and obviously it leads to a solid theory of online boosting. However, it has two main practical limitations. The first is the difficulty of estimating the edge $\gamma$. Given a learner and an adversary, it is by no means a simple task to find the maximum edge that satisfies (\ref{eq:onlineWLC}). The second issue is that different learners may have different edges. Some learners may in fact be quite strong with significant edges, while others are just slightly better than a random guess. In this case, OnlineMBBM has to pick the minimum edge as it assumes common $\gamma$ for all weak learners. It is obviously inefficient in that the booster underestimates the strong learners' accuracy. 

Our adaptive algorithm will discard the online weak learning condition to provide a more practical method. Empirical edges $\gamma_{1}, \cdots, \gamma_{N}$ (see Section~\ref{section:AdaboostOLM} for the definition) are measured for the weak learners and are used to bound the number of mistakes made by the boosting algorithm. 

\subsection{Choice of loss function}
\label{section:surrogate}
Adaboost, proposed by \citet{freund1999short}, is arguably the most popular boosting algorithm in practice. It aims to minimize the exponential loss, and has many variants which use some other surrogate loss. The main reason of using a surrogate loss is ease of optimization; while 0-1 loss is not even continuous, most surrogate losses are convex. We adopt the use of a surrogate loss for the same reason, and throughout this section will discuss our choice of surrogate loss for the adaptive algorithm.

Exponential loss is a very strong candidate in that it provides a closed form for computing potential functions, which are used to design cost matrices (cf. \citet[Theorem 13]{mukherjee2013theory}). One property of online setting, however, makes it unfavorable. Like OnlineMBBM, each data point will have a different weight depending on weak learners' performance, and if the algorithm uses exponential loss, this weight will be an exponential function of difference in weighted cumulative votes. With this exponentially varying weights among samples, the algorithm might end up depending on very small portion of observed samples. This is undesirable because it is easier for the adversary to manipulate the sample sequence to perturb the learner. 

To overcome exponentially varying weights, \citet{beygelzimer2015optimal} use logistic loss in their adaptive algorithm. Logistic loss is more desirable in that its derivative is bounded and thus weights will be relatively smooth. For this reason, we will also use multiclass version of logistic loss:
\begin{equation}
\label{eq:logisticLoss}
L^r (\cumv) =: \sum_{l \neq r} \log (1 + \exp(\cumv[r] - \cumv[r])).
\end{equation}
We still need to compute potential functions from logistic loss in order to calculate cost matrices. Unfortunately, \citet{mukherjee2013theory} use a unique property of exponential loss to get a closed form for potential functions, which cannot be adopted to logistic loss. However, the optimal cost matrix induced from exponential loss has a very close connection with the gradient of the loss (cf. \citet[Lemma 22]{mukherjee2013theory}). From this, we will design our cost matrices as following:
\begin{align}
\label{eq:costLogistic}
\cost^{i}_t[r,l] := 
    \begin{cases} 
    \frac{1}{1 + \exp(\cumv^{i-1}_t [r] - \cumv^{i-1}_t[l])}&, \text{if } l\neq r\\
    -\sum_{j\neq r}\frac{1}{1 + \exp(\cumv^{i-1}_t [r] - \cumv^{i-1}_t[j])}&, \text{if } l = r.
    \end{cases}
\end{align}
Readers should note that the row vector $\cost^{i}_{t}[r]$ is simply the gradient of $L^r(\cumv^{i-1}_{t})$. Also note that this matrix does not belong to $\reducedcostseor$, but it does guarantee that the correct prediction gets the minimal cost. 

The choice of logistic loss over exponential loss is somewhat subjective. The undesirable property of exponential loss does not necessarily mean that we cannot build an adaptive algorithm using this loss. In fact, we can slightly modify Algorithm \ref{alg:AdaboostOLM} to develop algorithms using different surrogates (exponential loss and square hinge loss). However, their theoretical bounds are inferior to the one with logistic loss. Interested readers can refer Appendix \ref{appendix:differentLosses}, but it assumes understanding of Algorithm \ref{alg:AdaboostOLM}.

\subsection{Adaboost.OLM}
\label{section:AdaboostOLM}
Our work is a generalization of Adaboost.OL by \citet{beygelzimer2015optimal}, from which the name Adaboost.OLM comes with M standing for multiclass. We introduce a new concept of an \textit{expert}. From $N$ weak learners, we can produce $N$ experts where expert $i$ makes its prediction by weighted majority votes among the first $i$ learners. Unlike OnlineMBBM, we allow varying weights $\alpha^{i}_{t}$ over the learners. As we are working with logistic loss, we want to minimize $\sum_{t}L^{y_{t}}(\cumv^{i}_{t})$ for each $i$, where the loss is given in (\ref{eq:logisticLoss}). We want to alert the readers to note that even though the algorithm tries to minimize the cumulative surrogate loss, its performance is still evaluated by 0-1 loss. The surrogate loss only plays a role of a bridge that makes the algorithm adaptive. 


We do not impose the online weak learning condition on weak learners, but instead just measure the performance of $WL^{i}$ by $\gamma_{i}:= \frac{\sum_{t}\cost^{i}_{t}[y_{t}, l^{i}_{t}]}{\sum_{t}\cost^{i}_{t}[y_{t}, y_{t}]}$. This \emph{empirical edge} will be used to bound the number of mistakes made by Adaboost.OLM. By definition of cost matrix, we can check
\begin{equation*}
\cost^{i}_{t}[y_{t}, y_{t}] \leq \cost^{i}_{t}[y_{t}, l] \leq -\cost^{i}_{t}[y_{t}, y_{t}], ~ \forall l \in [k],
\end{equation*}
from which we can prove $-1 \leq \gamma_{i} \leq 1, ~\forall i$. If the online weak learning condition is met with edge $\gamma$, then one can show that $\gamma_{i} \geq \gamma$ with high probability when the sample size is sufficiently large. 

Unlike the optimal algorithm, we cannot show the last expert that utilizes all the learners has the best accuracy. However, we can show at least one expert has a good predicting power. Therefore we will use classical \textit{Hedge algorithm} (\citet{littlestone1989weighted} and \citet{freund1995desicion}) to randomly choose an expert at each iteration with adaptive probability weight depending on each expert's prediction history. 

Finally we need to address how to set the weight $\alpha^{i}_{t}$ for each weak learner. As our algorithm tries to minimize the cumulative logistic loss, we want to set $\alpha^{i}_{t}$ to minimize $\sum_{t} L^{y_{t}}(\cumv^{i-1}_{t} + \alpha^{i}_{t} \stdv_{l^{i}_{t}})$. This is again a classical topic in online learning, and we will use \textit{online gradient descent}, proposed by \citet{zinkevich2003online}. By letting, $f^{i}_{t}(\alpha) := L^{y_{t}}(\cumv^{i-1}_{t}+\alpha \stdv_{l^{i}_{t}})$, we need an online algorithm ensuring $\sum_{t} f^{i}_{t} (\alpha^{i}_{t}) \leq \min_{\alpha \in F} \sum_{t}f^{i}_{t}(\alpha) + R^{i}(T)$ where $F$ is a feasible set to be specified later, and $R^{i}(T)$ is a regret that is sublinear in $T$. To apply \citet[Theorem 1]{zinkevich2003online}, we need $f^{i}_{t}$ to be convex and $F$ to be compact. The first assumption is met by our choice of logistic loss, and for the second assumption, we will set $F = [-2, 2]$. There is no harm to restrict the choice of $\alpha^{i}_{t}$ by $F$ because we can always scale the weights without affecting the result of weighted majority votes. 

By taking derivatives, we get
\begin{align}
\label{eq:gradientLogistic}
{f^i_t}'(\alpha) = 
    \begin{cases} 
    \frac{1}{1 + \exp(\cumv^{i-1}_t [y_t] - \cumv^{i-1}_t[l^i_t]-\alpha)}&, \text{if } l^i_t\neq y_t\\
    -\sum_{j\neq y_t}\frac{1}{1 + \exp(\cumv^{i-1}_t[j]+\alpha -\cumv^{i-1}_t [y_t])}&, \text{if } l^i_t = y_t.
    \end{cases}
\end{align}
This provides $|{f^i_t}'(\alpha)| \leq k-1$. Now let $\Pi(\cdot)$ represent a projection onto $F$: $\Pi(\cdot) := \max \{-2, \min\{2, \cdot\}\}$. By setting $\alpha^i_{t+1} = \Pi (\alpha^i_t - \eta_t {f^i_t}'(\alpha^i_t))$ where $\eta_{t} = \frac{2 \sqrt 2}{(k-1) \sqrt t}$, we get $R^{i}(T) \leq 4\sqrt 2 (k-1) \sqrt T$. Readers should note that any learning rate of the form $\eta_{t} = \frac{c}{\sqrt t}$ would work, but our choice is optimized to ensure the minimal regret. 

The pseudocode for Adaboost.OLM is presented in Algorithm \ref{alg:AdaboostOLM}. In fact, if we put $k=2$, Adaboost.OLM has the same structure with Adaboost.OL. As in OnlineMBBM, the booster also needs to pass the weight along with labeled instance. According to (\ref{eq:costLogistic}), it can be inferred that the weight is proportional to $-\cost^{i}_{t}[y_{t}, y_{t}]$. 

\begin{algorithm}[t!]
	\begin{algorithmic}[1]
		\STATE \textbf{Initialize:} $\forall i, v^{i}_{1} = 1, \alpha^{i}_{1} = 0$
		\FOR {$t = 1, \cdots, T$}
		\STATE Receive example $\samplex_{t}$
		\STATE Set $\cumv^{0}_{t} = \zerov \in \reals^{k}$
		\FOR {$i = 1, \cdots, N$}
		\STATE Compute $\cost^{i}_{t}$ according to (\ref{eq:costLogistic}) and pass it to $WL^{i}$
		\STATE Set $l^{i}_{t} = WL^{i}(\samplex_{t})$ and $\cumv^{i}_{t} = \cumv^{i-1}_{t} + \alpha^{i}_{t}\stdv_{l^{i}_{t}}$
		\STATE Set $\predy^{i}_{t} = \argmax_{l} \cumv^{i}_{t} [l]$, the prediction of expert $i$
		\ENDFOR
		\STATE Randomly draw $i_{t}$ with $\prob(i_{t} =i) \propto v^{i}_{t}$
		\STATE Predict $\predy_{t} = \predy^{i_{t}}_{t}$ and receive the true label $y_{t}$
		\FOR {$i = 1, \cdots, N$}
		\STATE Set $\alpha^i_{t+1} = \Pi (\alpha^i_t - \eta_t {f^i_t}'(\alpha^i_t))$ using (\ref{eq:gradientLogistic}) and $\eta_{t} = \frac{2 \sqrt 2}{(k-1) \sqrt t}$
		\STATE Set $\weightv^{i}[t] = -\frac{\cost^{i}_{t}[y_{t}, y_{t}]}{k-1}$ and pass $(\samplex_{t}, y_{t}, \weightv^{i}[t])$ to $WL^{i}$		\STATE Set $v^{i}_{t+1} = v^{i}_{t} \cdot \exp(-\ind(y_{t} \neq \predy^{i}_{t}))$
		\ENDFOR
		\ENDFOR
	\end{algorithmic}
	\caption{Adaboost.OLM}
	\label{alg:AdaboostOLM}
\end{algorithm}

\subsection{Mistake bound and comparison to the optimal algorithm}
\label{section:mistakeAdaptive}
Now we present our second main result that provides a mistake bound of Adaboost.OLM. The main structure of the proof is adopted from \citet[Theorem 4]{beygelzimer2015optimal} but in a generalized cost matrix framework. The proof appears in Appendix \ref{appendix:mistakeAdaptive}.

\begin{theorem}{\bf{(Mistake bound of Adaboost.OLM)}}
\label{thm:mistakeAdaptive}
For any $T$ and $N$, with probability $1-\delta$, the number of mistakes made by Adaboost.OLM satisfies the following inequality:

\begin{equation*}
\sum_{t=1}^{T} \ind(y_{t} \neq \predy_{t}) \leq \frac{8 (k-1)}{\sum_{i=1}^{N}\gamma_{i}^{2}}T + \tilde O (\frac{kN^{2}}{\sum_{i=1}^{N}\gamma_{i}^{2}}),
\end{equation*}

where $\tilde O$ notation suppresses dependence on $\log \frac{1}{\delta}$.
\end{theorem}

\begin{remark}
Note that this theorem naturally implies \citet[Theorem 4]{beygelzimer2015optimal}. The difference in coefficients is due to different scaling of $\gamma_{i}$. In fact, their $\gamma_{i}$ ranges from $[-\frac{1}{2}, \frac{1}{2}]$. 
\end{remark}

Now that we have established a mistake bound, it is worthwhile to compare the bound with the optimal boosting algorithm. Suppose the weak learners satisfy the weak learning condition (\ref{eq:onlineWLC}) with edge $\gamma$. For simplicity, we will ignore the excess loss $S$. As we have $\gamma_{i} = \frac{\sum_{t}\cost^{i}_{t}[y_{t}, l^{i}_{t}]}{\sum_{t}\cost^{i}_{t}[y_{t}, y_{t}]} \geq \gamma$ with high probability, the mistake bound becomes $\frac{8 (k-1)}{\gamma^{2}N}T + \tilde O (\frac{kN}{\gamma^{2}})$. In order to achieve error rate $\epsilon$, Adaboost.OLM requires $N\geq \frac{8(k-1)}{\epsilon \gamma^{2}}$ learners and $T = \tilde \Omega(\frac{k^{2}}{\epsilon^{2} \gamma^{4}})$ sample size. Note that OnlineMBBM requires $N = \Omega (\frac{1}{\gamma^{2}}\ln \frac{k}{\epsilon})$ and $T = \min\{\tilde \Omega (\frac{k^{5/2}}{\epsilon \gamma}),~ \tilde \Omega(\frac{k}{\epsilon \gamma ^{2}})\}$. Adaboost.OLM is obviously suboptimal, but due to its adaptive feature, its performance on real data is quite comparable to that by OnlineMBBM. 

\section{Experiments}
\label{section:experiment}
We compare the new algorithms to existing ones for online boosting on several UCI data sets, each with $k$ classes\footnote{Codes are available at \url{https://github.com/yhjung88/OnlineBoostingWithVFDT}}. Table \ref{tab:Comparisons} contains some highlights, with additional results and experimental details in the Appendix \ref{appendix:experiment}. Here we show both the average accuracy on the final 20\% of each data set, as well as the average run time for each algorithm. Best decision tree gives the performance of the best of 100 online decision trees fit using the VFDT algorithm in \citet{domingos2000mining}, which were used as the weak learners in all other algorithms, and Online Boosting is an algorithm taken from \citet{oza2005online}. Both provide a baseline for comparison with the new Adaboost.OLM and OnlineMBBM algorithms. Best MBBM takes the best result from running the OnlineMBBM with five different values of the edge parameter $\gamma$.

Despite being theoretically weaker, Adaboost.OLM often demonstrates similar accuracy and sometimes outperforms Best MBBM, which exemplifies the power of adaptivity in practice. This power comes from the ability to use diverse learners efficiently, instead of being limited by the strength of the weakest learner. OnlineMBBM suffers from high computational cost, as well as the difficulty of choosing the correct value of $\gamma$, which in general is unknown, but when the correct value of $\gamma$ is used it peforms very well. Finally in all cases Adaboost.OLM and OnlineMBBM algorithms outperform both the best tree and the preexisting Online Boosting algorithm, while also enjoying theoretical accuracy bounds. 

\begin{table}[h]
    \caption{Comparison of algorithm accuracy on final 20\% of data set and run time in seconds. Best accuracy on a data set reported in \textbf{bold}.}
    \label{tab:Comparisons}
    \centering
    \setlength\tabcolsep{3pt}
    \begin{tabular}{lcccrccrccrccrc}
        \toprule
        Data sets & $k$&& \multicolumn{2}{c}{Best decision tree} & & \multicolumn{2}{c}{Online Boosting} & & \multicolumn{2}{c}{Adaboost.OLM} & & \multicolumn{2}{c}{Best MBBM}\\
        \midrule
        Balance & 3 &&  0.768 & 8 && 0.772 & 19 && 0.754 & 20 && \textbf{0.821} & 42 \\ 
        Mice & 8 && 0.608 & 105 && 0.399 & 263 && 0.561 & 416 && \textbf{0.695} & 2173 \\
        Cars & 4 && 0.924 & 39 && 0.914 & 27 && \textbf{0.930} & 59 && 0.914 & 56 \\
        Mushroom & 2 && 0.999 & 241 && \textbf{1.000} & 169 && \textbf{1.000} & 355 && \textbf{1.000} & 325\\
        Nursery & 4 && 0.953 & 526 && 0.941 & 302 && 0.966 & 735 && \textbf{0.969} & 1510 \\
        ISOLET & 26 && 0.515 & 470 && 0.149 & 1497 && 0.521 & 2422 && \textbf{0.635} & 64707 \\
        Movement & 5 && 0.915 & 1960 && 0.870 & 3437 && 0.962 & 5072 && \textbf{0.988} & 18676 \\
        \bottomrule
    \end{tabular}
\end{table}

 \subsubsection*{Acknowledgments}
We acknowledge the support of NSF under grants CAREER IIS-1452099 and CIF-1422157. 

\small \bibliography{myref} 
\normalsize

\newpage
\begin{appendices}
\section{Link between batch and online weak learning conditions}
\label{appendix:linkWLC}
Let us begin the section by introducing the weak learning condition in the batch setting. \citet{mukherjee2013theory} have identified necessary and sufficient condition for boostability. We will focus on a sufficient condition due to reasons of computational tractability. In the batch setting, the entire training set is revealed. Let $D:=\{(\samplex_{t}, y_{t}) ~|~ t = 1, \cdots, T\}$ be the training set and define a family of cost matrices: 
\begin{equation*}
\costseor := \{\cost \in \reals^{T \times k} ~|~ \forall t, ~\cost[t, y_{t}] = \min_{l \in [k]}\cost[t, l]\}.  
\end{equation*}
The superscript ``eor'' stands for ``edge-over-random.'' We warn the readers not to confuse $\costseor$ with $\reducedcostseor$. They both impose similar row constraints, but the matrices in these sets have different dimensions: $T \times k$ and $k \times k$ respectively. $\reducedcostseor$ also has additional an normalization constraint. Note that $\costseor$ provides one cost vector for an instance whereas $\reducedcostseor$ provides a matrix. This is necessary because if an adversary passes only a vector to an online learner, then the learner can simply make the prediction which minimizes the cost. Furthermore, in the online boosting setting, the booster does not know the true label when it computes a cost matrix. 

The authors prove that if a weak learning space $\weakspace$ satisfies the condition described in Definition \ref{def:batchWLC}, then it is boostable, which means there exists a convex linear combination of hypotheses in $\weakspace$ that perfectly classifies $D$. 

\begin{definition}{\bf{(Batch setting weak learning condition, \citet{mukherjee2013theory})}}
\label{def:batchWLC}
Suppose $D$ is fixed and $\costseor$ is defined as above. A weak learning space $\weakspace$ is said to satisfy weak learning condition $(\costseor, \unifm_{\gamma})$ if $\forall \cost \in \costseor$, one can find a weak hypothesis $h \in \weakspace$ such that
\begin{equation}
\sum_{t=1}^{T} \cost[t, h(\samplex_{t})] \leq \cost \bullet \unifm_{\gamma}'.
\label{eq:batchWLC}
\end{equation}
\end{definition}	

Now we present how our online weak learning condition (Definition \ref{def:onlineWLC}) is naturally derived from the batch setting counterpart (Definition \ref{def:batchWLC}). We extend the arguments of \citet{beygelzimer2015optimal}. The batch setting condition (\ref{eq:batchWLC}) can be interpreted as making the following two implicit assumptions: 
\begin{enumerate}
	\item (Richness condition) For any $\cost \in \costseor$, there is some hypothesis $h \in \weakspace$ such that 
	\begin{equation*}
		\label{eq:batchRich}
		\sum_{t=1}^{T}\cost[t, h(\samplex_{t})] \leq \cost \bullet \unifm_{\gamma}'.
	\end{equation*}
	\item (Agnostic learnability) For any $\cost \in \costseor$ and $\epsilon \in (0, 1)$, there is an algorithm which can compute a nearly optimal hypothesis $h \in \weakspace$, i.e.
	\begin{equation*}
		\sum_{t=1}^{T}\cost[t, h(\samplex_{t})] \leq \inf_{h' \in \weakspace}\sum_{t=1}^{T}\cost[t, h'(\samplex_{t})] + \epsilon T.
	\end{equation*}
\end{enumerate}

For the online setting, we will keep the richness assumption with $\cost$ being the matrix consisting of rows of $w_{t}\cost_{t}[y_{t}]$, and the data being drawn by a fixed adversary. That is to say, it is the online richness condition that imposes a restriction on adversary because the condition cannot be met by any $\weakspace$ with fully adaptive adversary. For example, suppose an adversary draws samples uniformly at random from the set $\{ (\samplex, 1), \cdots, (\samplex,k)\}$ for some fixed $\samplex \in \mathcal{X}$. There does not exist weak learning space $\weakspace$ that satisfies the online richness condition with this adversary. The agnostic learnability assumption is also replaced by online agnostic learnability assumption. We present online versions of the above two assumptions:
\begin{enumerate}
	\item[1$^{\prime}$.] (Online richness condition) For any sample length $T$, any sequence of labeled examples $\{(\samplex_{t}, y_{t}) ~|~t=1, \cdots, T\}$ generated by a fixed adversary, and any series of pairs of weight and cost matrix $\{(w_{t}, \cost_{t}) \in [0, 1] \times \reducedcostseor~|~t = 1, \cdots, T\}$, there is some hypothesis $h \in \weakspace$ such that 
	\begin{equation}
		\label{eq:rich}
		\sum_{t=1}^{T}w_{t}\cost_{t}[y_{t}, h(\samplex_{t})] \leq \cost \bullet \unifm_{\gamma}',
	\end{equation}
	where $\cost \in \reals^{T \times k}$ consists of rows of $w_{t}\cost_{t}[y_{t}]$. 
	\item[2$^{\prime}$.] (Online agnostic learnability) For any sample length $T$, $\delta \in (0, 1)$, and for any adaptively chosen series of pairs of weight and cost matrix $\{(w_{t}, \cost_{t}) \in [0, 1] \times \reducedcostseor~|~t = 1, \cdots, T\}$, there is an online algorithm which can generate predictions $\predy_{t}$ such that with probability $1-\delta$,
	\begin{equation}
		\label{eq:onlineAgnosticLearnability}
		\sum_{t=1}^{T}w_{t}\cost_{t}[y_{t}, \predy_{t}] \leq \inf_{h \in \weakspace}\sum_{t=1}^{T}w_{t}\cost_{t}[y_{t}, h(\samplex_{t})] + R_{\delta}(T),
	\end{equation}
	where $R_{\delta}:\naturals \rightarrow \reals$ is a sublinear regret. 
\end{enumerate}

\citet{daniely2011multiclass} extensively investigates agnostic learnability in online multiclass problems by introducing the following generalized Littlestone dimension (\citet{littlestone1988learning}) of a hypothesis family $\weakspace$. Consider a binary rooted tree $RT$ whose internal nodes are labeled by elements from $\mathcal{X}$ and whose edges are labeled by elements from $[k]$ such that two edges from a same parent have different labels. The tree $RT$ is \textit{shattered} by $\weakspace$ if, for every path from root to leaf which traverses the nodes $\samplex_{1}, \cdots, \samplex_{k}$, there is a hypothesis $h \in \weakspace$ such that $h(\samplex_{i})$ corresponds to the label of the edge from $\samplex_{i}$ to $\samplex_{i+1}$. The \textit{Littlestone dimension} of $\weakspace$ is the maximal depth of complete binary tree that is shattered by $\weakspace$ (or $\infty$ if one can build a arbitrarily deep shattered tree). The authors prove that an optimal online algorithm has a sublinear regret under the expected (w.r.t. the randomness of the algorithm) 0-1 loss if Littlestone dimension of $\weakspace$ is finite. 

Similarly we prove in Lemma \ref{lemma:onlineAgnosticLearnability} that the condition (\ref{eq:onlineAgnosticLearnability}) is satisfied if $\weakspace$ has a finite Littlestone dimension. We need to slightly modify their result in two ways. One is to replace expectation by probabilistic argument, and the other is to replace 0-1 loss by our cost matrix framework. Both questions can be resolved by replacing an auxiliary lemma used by \citet{daniely2011multiclass} without changing the main structure. 

\begin{lemma}
	\label{lemma:onlineAgnosticLearnability}
	Suppose a weak learning space $\weakspace$ has a finite Littlestone dimension $d$ and an adversary chooses examples in fully adaptive manner. For any sample length $T$ and for any adaptively chosen series of pairs of weight and cost matrix $\{(w_{t}, \cost_{t}) \in [0, 1] \times \reducedcostseor~|~t = 1, \cdots, T\}$, with probability $1-\delta$, the online agnostic learnability condition (\ref{eq:onlineAgnosticLearnability}) is satisfied with following sublinear regret
	\begin{equation*}
		R_{\delta}(T) = \sqrt{(Td\ln Tk)/2} + \sqrt{(T\ln 1/\delta)/2}.
	\end{equation*}
\end{lemma}

\begin{proof}
    	We first introduce an online algorithm with experts. Suppose we have a fixed pool of experts of size $N$. We keep our cost matrix framework. Each expert $f^{i}$ would suffer cumulative cost $C^{i}_{T} := \sum_{t=1}^{T}w_{t}\cost_{t}[y_{t}, f^{i}(\samplex_{t})]$. At each iteration, an online algorithm chooses to follow one expert and incurs a cost $w_{t}\cost_{t}[y_{t}, \predy_{t}]$, and its goal is to perform as well as the best expert. That is to say, the algorithm wants to keep its cumulative cost $\sum_{t=1}^{T}w_{t}\cost_{t}[y_{t}, \predy_{t}]$ not too much larger than $\min_{i \in [N]}C^{i}_{T}$. This learning framework is called \textit{weighted majority algorithm} and is thoroughly investigated by several researchers (e.g., \citet{littlestone1989weighted} and \citet{vovk1990aggregating}). We will specifically use Algorithm \ref{alg:LEA} (LEA), which is shown to achieve a sublinear regret $\sqrt{(T\ln N)/2} + \sqrt{(T\ln 1/\delta)/2}$ with probability $1-\delta$ (cf. \citet[Corollary 4.2]{cesa2006prediction}). The authors require the loss to be bounded, which is also satisfied in our cost matrix framework. Readers might raise a question that our loss function changes for each iteration, but the proof still works as long as it is bounded. Interested readers might refer \citet[Section 1.3.3]{hazan2016introduction}. 
    
    	To apply this result in our case, we need to construct a finite set of experts whose best performance is as good as that of hypotheses in $\weakspace$. In fact, in the proof of \citet[Theorem 25]{daniely2011multiclass}, the authors construct a set $E$ of size $N \leq (Tk)^{d}$ such that for every hypothesis $h \in \weakspace$, there is an expert $f \in E$ which coincides with $h$ subject to the given examples $\samplex_{1}, \cdots, \samplex_{T}$. 

	Applying the LEA result on $E$ shows that with probability $1-\delta$, the regret is bounded above by $\sqrt{(Td\ln Tk)/2} + \sqrt{(T\ln 1/\delta)/2}$, which concludes the proof. 
\end{proof}

\begin{algorithm}[t!]
	\begin{algorithmic}[1]
		\STATE \textbf{Input} T: time horizon, N: number of experts
		\STATE Set $\eta = \sqrt{(8\ln N) /T}$
		\STATE Set $C^{i}_{0} = 0$ for all $i$
		\FOR{$t=1, \cdots, T$}
		\STATE Receive example $\samplex_t$
		\STATE Receive expert advices $(f^{1}_{t}, \cdots, f^{N}_{t}) \in [k]^{N}$
		\STATE Predict $\predy_{t} = f^{i}_{t}$ with probability proportional to $\exp(-\eta C^{i}_{t-1})$
		\STATE Receive true label $y_{t}$ 
		\STATE Update $C^{i}_{t} = C^{i}_{t-1} + w_{t}\cost_{t}[y_{t}, f^{i}_{t}]$ for all $i$
		\ENDFOR
	\end{algorithmic}
	\caption{Learning with Expert Advice (LEA)}
	\label{alg:LEA}
\end{algorithm}

One remark is that the proof of Lemma \ref{lemma:onlineAgnosticLearnability} only uses the boundedness condition of $\reducedcostseor$.

Now we are ready to demonstrate that our online weak learning condition is indeed naturally derived from the batch setting counterpart. The following Theorem shows that two conditions (\ref{eq:rich}) and (\ref{eq:onlineAgnosticLearnability}) directly imply the online weak learning condition (\ref{eq:onlineWLC}). In other words, if the weak learning space $\weakspace$ accompanied by an adversary is rich enough to contain a hypothesis that slightly outperforms a random guess and has a reasonably small dimension, then we can find an excess loss $S$ that satisfies (\ref{eq:onlineWLC}). This is a generalization of \citet[Lemma 2]{beygelzimer2015optimal}. Note that we impose an additional assumption that $w_{t} \geq m>0 ~,~\forall t$. In case the learner encounters zero weight, it can simply ignore the instance, and the above assumption is not too artificial. 

\begin{theorem}{\bf{(Link between batch and online weak learning conditions)}}
	\label{thm:linkWLC}
	Suppose a pair of weak learning space $\weakspace$ and an adversary satisfies online richness assumption (\ref{eq:rich}) with edge $2\gamma$ and online agnostic learnability assumption (\ref{eq:onlineAgnosticLearnability}) with mistake probability $\delta$ and sublinear regret $R_{\delta}(\cdot)$. Additionally we assume there exists a positive constant $m$ that satisfies $w_{t} \geq m ~,~\forall t$. Then the online learning algorithm satisfies the online weak learning condition (\ref{eq:onlineWLC}), with mistake probability $\delta$, edge $\gamma$, and excess loss $S = \max_{T}(R_{\delta}(T) - \frac{\gamma mT}{k})$. 
\end{theorem}
\begin{proof}
Fix $\delta \in (0, 1)$ and a series of pairs of weight and cost matrix $\{(w_{t}, \cost_{t}) \in [0, 1] \times \reducedcostseor ~|~ t = 1, \cdots, T\}$, and let $\cost \in \reals^{T \times k}$ consist of rows of $w_{t}\cost_{t}[y_{t}]$. First note that by sublinearity of $R_{\delta}(\cdot)$, $S$ is finite. According to (\ref{eq:onlineAgnosticLearnability}), the online learning algorithm can generate predictions $\predy_{t}$ such that, with probability $1-\delta$, 
\begin{equation*}
	\sum_{t=1}^{T} w_{t} \cost_{t}[y_{t}, \predy_{t}] \leq \cost \bullet \unifm_{2\gamma}' + R_{\delta}(T).
\end{equation*}
Thus it suffices to show that 
\begin{equation}
	\label{eq:WLC1}
	\cost \bullet \unifm_{2\gamma}' + R_{\delta}(T) \leq \cost \bullet \unifm_{\gamma}' + S.
\end{equation}
Since the correct label gets zero cost and the row $\cost[r]$ has $\ell_1$ norm $w_{t}$, we have 
\begin{equation*}
\cost \bullet \unifm_{\gamma}' = \frac{1-\gamma}{k} ||\cost||_{1} =  \frac{1-\gamma}{k}\sum_{t=1}^{T}w_{t}.
\end{equation*}
By plugging this in (\ref{eq:WLC1}), we get 
\begin{equation*}
	\cost \bullet \unifm_{2\gamma}' - \cost \bullet \unifm_{\gamma}' + R_{\delta}(T) = 
	-\frac{\gamma}{k} \sum_{t=1}^{T}w_{t} + R_{\delta}(T) \leq -\frac{\gamma}{k} mT + R_{\delta}(T) \leq S.
\end{equation*}
The first inequality holds because $w_{t} \geq m$, and the second inequality holds by definition of $S$, which completes the proof. 
\end{proof}

Lemma \ref{lemma:onlineAgnosticLearnability} and Theorem \ref{thm:linkWLC} suggest an implicit relation between $\delta$ and $S$ in (\ref{eq:onlineWLC}). If we want probabilistically stronger weak learning condition, $R_{\delta}(T)$ in Lemma \ref{lemma:onlineAgnosticLearnability} gets bigger, which results in larger $S = \max_{T}(R_{\delta}(T) - \frac{\gamma T}{k})$. 

\section{Detailed discussion of OnlineMBBM}
\subsection{Proof of Theorem \ref{thm:mistakeOptimal}}
\label{appendix:mistakeOptimal}
\begin{proof}
	For ease of notation, we will assume the edge is equal to $\gamma$ and the true label is $r$ unless otherwise specified. That is to say, $\unifv$ stands for $\unifv^{r}_{\gamma}$ and $\potential_{i}$ for $\potential^{r}_{i}$. By rewriting (\ref{eq:potentialCond}),
	\begin{align*}
		\potential_{N-i+1}(\cumv^{i-1}_{t}) &= \E_{l \sim \unifv} \potential_{N-i}(\cumv^{i-1}_{t}+ \stdv_{l}) \\
		&= \cost^{i}_{t}[r] \bullet \unifv \\ 
		&= \cost^{i}_{t}[r] \bullet (\unifv - \stdv_{l^{i}_{t}}) + \potential_{N-i}(\cumv^{i}_{t}),
	\end{align*}
	where $\cost^{i}_{t}$ is defined in (\ref{eq:costMatrix}). The last equation holds due to the relation $\cumv^{i}_{t} = \cumv^{i-1}_{t} + \stdv_{l^{i}_{t}}$. Also note that $||\unifv||_{1} = ||\stdv_{r}||_{1} =1$, and thus subtracting common numbers from each component of $\cost^{i}_{t}[r]$ does not affect the dot product term. Therefore, by introducing normalized cost matrix $\normcost^{i}_{t}$ as in (\ref{eq:normalizedCostMatrix}) and $\weightv^{i}[t]$ as in Algorithm \ref{alg:onlineMBBM}, we may write
	\begin{align}
		\begin{split}
		\label{eq:mistake1}
		\potential^{y_{t}}_{N-i+1}(\cumv^{i-1}_{t}) &= \weightv^{i}[t] \normcost^{i}_{t}[y_{t}] \bullet (\unifv^{y_{t}}_{\gamma} - \stdv_{l^{i}_{t}}) + \potential^{y_{t}}_{N-i}(\cumv^{i}_{t}) \\
		&= \weightv^{i}[t] \normcost^{i}_{t}[y_{t}] \bullet \unifv^{y_{t}}_{\gamma} - \weightv^{i}[t] \normcost^{i}_{t}[y_{t}, l^{i}_{t}] + \potential^{y_{t}}_{N-i}(\cumv^{i}_{t}) \\
		&= \weightv^{i}[t] \frac{1-\gamma}{k} - \weightv^{i}[t] \normcost^{i}_{t}[y_{t}, l^{i}_{t}] + \potential^{y_{t}}_{N-i}(\cumv^{i}_{t}).
		\end{split}
	\end{align}
	The last equality holds because $\normcost^{i}_{t}$ is normalized and $\normcost^{i}_{t}[y_{t}, y_{t}] = 0$. If $\normcost^{i}_{t}[y_{t}]$ is a zero vector, then by definition $\weightv^{i}[t] = 0$, and the equality still holds. Then by summing (\ref{eq:mistake1}) over $t$, we get
	\begin{equation*}
		\sum_{t=1}^{T} \potential^{y_{t}}_{N-i+1}(\cumv^{i-1}_{t}) = \frac{1-\gamma}{k}||\weightv^{i}||_{1} - \sum_{t=1}^{T} \weightv^{i}[t]\normcost^{i}_{t}[y_{t}, l^{i}_{t}] + \sum_{t=1}^{T}\potential^{y_{t}}_{N-i}(\cumv^{i}_{t}).
	\end{equation*}
	By online weak learning condition, we have with probability $1-\delta$, (recall that $w^{i*}$ estimates $||\weightv^{i}||_{\infty}$)
	\begin{equation*}
	\sum_{t=1}^{T} \frac{\weightv^{i}[t]}{w^{i*}}\normcost^{i}_{t}[y_{t}, l^{i}_{t}] \leq \frac{1-\gamma}{k}\frac{||\weightv^{i}||_{1}}{w^{i*}} + S.
	\end{equation*}
	From this, we can argue that
	\begin{equation*}
		\sum_{t=1}^{T} \potential^{y_{t}}_{N-i+1}(\cumv^{i-1}_{t}) + S w^{i*} \geq \sum_{t=1}^{T}\potential^{y_{t}}_{N-i}(\cumv^{i}_{t}).
	\end{equation*}
	Since the above inequality holds for any $i$, summing over $i$ gives
	\begin{equation*}
		\sum_{t=1}^{T} \potential^{y_{t}}_{N}(\textbf{0}) + S \sum_{i=1}^{N} w^{i*} \geq \sum_{t=1}^{T}\potential^{y_{t}}_{0}(\cumv^{N}_{t}),
	\end{equation*}
	which holds with probability $1 - N\delta$ by union bound. 
	By symmetry, $\potential^{y_{t}}_{N}(\textbf{0}) = \potential^{1}_{N}(\textbf{0})$ regardless of the true label $y_{t}$, and by definition of potential function (\ref{eq:potentialCond}), $\potential^{y_{t}}_{0}(\cumv^{N}_{t}) = L^{y_{t}}(\cumv^{N}_{t})$, which completes the proof.
	
\end{proof}

\subsection{Bounding the terms in general bound under 0-1 loss}
\label{appendix:mistakeSpecific}
Even though OnlineMBBM has a promising theoretical justification, it would be infeasible if the computation of potential functions takes too long or if the behavior of asymptotic error rate $\potential^{1}_{N}(\textbf{0})$ is too complicated to be approximated. Fortunately for the 0-1 loss, we can get a computationally tractable algorithm with vanishing error rate. The use of potential functions in binary boosting setup is thoroughly discussed by \citet{schapire2001drifting}. In binary setting under 0-1 loss, potential function has a closed form which dramatically reduces the computational complexity. Unfortunately, the multiclass version does not have a closed form, but \citet{mukherjee2013theory} introduce a heuristic to compute it in reasonable time:
\begin{equation}
	\label{eq:potentialCompute}
	\potential^{r}_{i}(\cumv) = 1 - \sum_{(x_{1}, \cdots, x_{k}) \in A} {i\choose x_{1}, \cdots, x_{k}} \prod_{l=1}^{k}u_{l}^{x_{l}},
\end{equation}
where $A := \{(x_{1}, \cdots x_{k}) \in \integers^{k} ~|~ x_{1} + \cdots x_{k} = i, ~ \forall l ~:~ x_{l} \geq 0, x_{l} + \cumv[l] < x_{r} + \cumv[r]\}$, and $\unifv^{r}_{\gamma} = (u_{1}, \cdots, u_{k})$. By using dynamic programming, the RHS of \eqref{eq:potentialCompute} can be computed in polynomial time in $i$, $k$, and $||\cumv||_{1}$. In our setting where the number of learners is fixed to be $N$, the computation can be done in polynomial time in $k$ and $N$ because $||\cumv||_{1}$ is bounded by $N$. To the best of our knowledge, there is no way to compute the potential function in polynomial time if we start from necessary and sufficient weak learning condition (the algorithm given by \citet{mukherjee2013theory} takes exponential time in the number of learners), and this is the main reason that we use the sufficient condition. 
Recall from (\ref{eq:mistakeGeneralOptimal}) that $\potential^{1}_{N}(\textbf{0})$ plays a role of asymptotic error rate and the second term determines the sample complexity. The following two lemmas provide bounds for both terms.

By applying the Hoeffding's inequality, we can prove in Lemma \ref{lemma:asymptoticError} that $\potential^{1}_{N}(\textbf{0})$ vanishes exponentially fast as $N$ grows. That is to say, to get a satisfactory accuracy, we do not need too many learners. We also note that we can decide $N$ before the learning process begins, which is logically plausible. 
\begin{lemma}
	\label{lemma:asymptoticError}
	Under the same setting as in Theorem \ref{thm:mistakeOptimal} but with the particular choice of 0-1 loss, we may bound $\potential^{1}_{N}(\textbf{0})$ as follows:
	\begin{equation}
		\label{eq:asymptoticError}
		\potential^{1}_{N}(\textbf{0}) \leq (k-1) \exp(-\frac{\gamma^{2}N}{2}) .
	\end{equation}
\end{lemma}

\begin{proof}
	We reinterpret $\potential^{1}_{N}(\textbf{0})$ in (\ref{eq:potentialCompute}). Imagine that we draw numbers $N$ times from $[k]$ where the probability that a number $i$ is drawn is $\unifv^{1}_{\gamma}[i]$. That is to say, $1$ has highest probability of $\frac{1-\gamma}{k} + \gamma$, and other numbers have equal probability of $\frac{1-\gamma}{k}$. Then $\potential^{1}_{N}(\textbf{0})$ can be interpreted as a probability that the number that is drawn for the most time out of $N$ draws is not $1$. Let $A_{i}$ denote the event that the number $i$ gets more votes than the number $1$. Then we have by union bound,
	\begin{align}
		\begin{split}
		\label{eq:asymptoticError1}
		\potential^{1}_{N}(\textbf{0}) &= \prob (A_{2} \cup \cdots \cup A_{k}) \\
		&\leq \sum_{l=2}^{k}\prob(A_{i}) \\
		&= (k-1) \prob(A_{2})
		\end{split}
	\end{align}
	The last equality holds by symmetry. To compute $\prob (A_{2})$, imagine that we draw $1$ with probability $\frac{1-\gamma}{k} + \gamma$, $-1$ with probability $\frac{1-\gamma}{k} $, and $0$ otherwise. $\prob(A_{2})$ is equal to the probability that after independent $N$ draws, the summation of $N$ i.i.d. random numbers is non-positive. Thus by the Hoeffding's inequality, we get 
	\begin{equation}
		\label{eq:asymptoticError2}
		\prob(A_{2}) \leq \exp(-\frac{\gamma^{2}N}{2})
	\end{equation}
	Combining (\ref{eq:asymptoticError1}) and (\ref{eq:asymptoticError2}) completes the proof. 
\end{proof}

Now we have fixed $N$ based on the desired asymptotic accuracy. Since  0-1 loss is bounded in $[0, 1]$, so are potential functions. Then by definition of weights (cf. Algorithm \ref{alg:onlineMBBM}), $||\weightv^{i}||_{\infty}$ is trivially bounded above by $k$, which means we can use $w^{i*} = k ~~\forall i$. Thus the second term of (\ref{eq:mistakeGeneralOptimal}) is bounded above by $kNS$, which is valid. However, Lemma \ref{lemma:weightBound} allows a tighter bound. 

\begin{lemma}
	\label{lemma:weightBound}
	Under the same setting as in Theorem \ref{thm:mistakeOptimal} but with the particular choice of 0-1 loss and an additional constraint of $\gamma < \frac{1}{2}$, we may bound $||\weightv^{i}||_{\infty}$ by
	\begin{equation}
		\label{eq:weightBound}
		||\weightv^{i}||_{\infty} \leq \frac{ck^{5/2}}{\sqrt {N-i}},
	\end{equation}
	where $c$ is a universal constant that can be determined before the algorithm begins.
\end{lemma}

\begin{proof}
	We will start by providing a bound on $\potential^{r}_{m}(\cumv + \stdv_{l}) - \potential^{r}_{m}(\cumv + \stdv_{r})$. First note that it is non-negative as potential functions are proper. Again by using random draw framework as in the proof of Lemma \ref{lemma:asymptoticError} (now $r$ has the largest probability to be drawn), this value corresponds to the probability that after $m$ draws, the number $r$ wins the majority votes if the count starts from $\cumv + \stdv_{r}$ but loses if the count starts from $\cumv + \stdv_{l}$. Let $X_{1}, \cdots , X_{k}$ denote the number of draws of each number out of $m$ draws and define the events $A_{l} := \{ (X_{r} + \cumv[r])  - (X_{l} + \cumv[l])\in \{0, 1\}\}$. Then it can be checked that 
	\begin{align}
		\begin{split}
		\label{eq:weightBound1}
		&\potential^{r}_{m}(\cumv + \stdv_{l}) - \potential^{r}_{m}(\cumv + \stdv_{r}) \\
		&= \prob(\exists l' \text{ s.t. } X_{l'}+\cumv[l'] + \stdv_{l}[l'] \geq X_{r} + \cumv[r]) - \prob(\exists l' \text{ s.t. }  X_{l'} + \cumv[l'] \geq X_{r}+\cumv[r] + 1) \\
		&\leq \prob(\exists l' \text{ s.t. } X_{l'}+\cumv[l'] + \stdv_{l}[l'] \geq X_{r} + \cumv[r] \text{ and } \forall l', ~ X_{r} + \cumv[r] \geq X_{l'} + \cumv[l'])\\
		&\leq \prob(\exists l' \text{ s.t. } X_{l'}+\cumv[l'] + \stdv_{l}[l'] \geq X_{r} + \cumv[r] \geq X_{l'} + \cumv[l'])\\
		&= \prob (\bigcup_{l \neq r} A_{l}) \leq \sum_{l \neq r} \prob(A_{l}).
		\end{split}
	\end{align}
	The first inequality holds by $\prob(A) - \prob(B) \leq \prob(A-B)$. Individual probabilities can be written as
	\begin{align}
		\begin{split}
		\label{eq:weightBound2}
		\prob(A_{l}) &= \prob(X_{r} - X_{l} = \cumv[l] - \cumv[r]) + \prob(X_{r} - X_{l} = \cumv[l] - \cumv[r] + 1) \\
		&\leq 2 \max_{n} \prob(X_{r} - X_{l} = n).
		\end{split}
	\end{align}
	We can prove by applying the Berry-Esseen theorem that the last probability is $O(\frac{1}{\sqrt m})$. Let $Y_{1}, \cdots, Y_{m}$ be a sequence of i.i.d. random variables such that $Y_{j} \in \{-1, 0, 1\}$ and
	\begin{align*}
		\prob(Y_{j} = 1) &= \frac{1-\gamma}{k} + \gamma, \\
		 \prob(Y_{j} = -1) &=\frac{1-\gamma}{k}.			
	\end{align*}
	Note that $\E Y_{j} = \gamma$ and $Var(Y_{j}) = \frac{2(1-\gamma)}{k} + \gamma(1-\gamma) =:\sigma^{2}$. It can be easily checked that $Y := \sum_{j=1}^{m}Y_{j}$ has same distribution with $X_{r}-X_{l}$. Now we approximate $Y$ by a Gaussian random variable $W \sim N(m\gamma, m\sigma^{2})$. Let $F_{W}$ and $F_{Y}$ denote CDF of $W$ and $Y$, respectively, and let $f$ denote the density of $W$. First note that 
	\begin{align*}
		|\prob(Y=n) - \int_{n-1}^{n} f(w) dw| 
		&= |(F_{Y}(n) - F_{Y}(n-1)) - (F_{W}(n) - F_{W}(n-1))|\\
		&\leq |F_{Y}(n) - F_{W}(n)| + |F_{Y}(n-1) - F_{W}(n-1)|.
	\end{align*}
	We can apply the Berry-Esseen theorem to the last CDF differences, which provides
	\begin{equation}
		\label{eq:weightBound3}
		|\prob(Y=n) - \int_{n-1}^{n} f(w) dw| \leq \frac{2C\rho}{\sigma^{3}\sqrt m},
	\end{equation}
	where $C$ is the universal constant that appears in Berry-Esseen and $\rho := \E |Y_{j} - \gamma|^{3}$. As $Y_{j}$ is a bounded random variable, we have 
	\begin{equation*}
		\rho = \E |Y_{j}-\gamma|^{3} \leq (1+\gamma) \E|Y_{j}-\gamma|^{2} = (1+\gamma) \sigma^{2} \leq 2 \sigma^{2}.
	\end{equation*}
	Plugging this in (\ref{eq:weightBound3}) gives 
	\begin{equation*}
		|\prob(Y=n) - \int_{n-1}^{n} f(w) dw| \leq \frac{4C}{\sigma \sqrt m}
	\end{equation*}
	By simple algebra, we can deduce
	\begin{align}
	\begin{split}
		\label{eq:weightBound4}
		\prob (Y=n) &\leq  \int_{n-1}^{n} f(w) dw + \frac{4C}{\sigma \sqrt m} \\
		&\leq \sup _{w \in \reals} f(w) + \frac{4C}{\sigma \sqrt m} \\
		& = \frac{1}{\sqrt{2\pi m}\sigma} + \frac{4C}{\sigma \sqrt m}.
	\end{split}
	\end{align}
	Using the fact that $\gamma < \frac{1}{2}$, we can show
	\begin{equation*}
		\sigma^{2} =  \frac{2(1-\gamma)}{k} + \gamma(1-\gamma) \geq \frac{1}{k}
	\end{equation*}
	Plugging this in (\ref{eq:weightBound4}) gives 
	\begin{equation}
		\label{eq:weightBound5}
		\prob(Y=n) \leq \frac{1}{\sigma \sqrt m} (\frac{1}{\sqrt{2\pi}} + 4C) \leq C'\sqrt \frac{k}{m},
	\end{equation}
	where $C' = \frac{1}{\sqrt{2\pi}} + 4C$. By combining (\ref{eq:weightBound1}), (\ref{eq:weightBound2}), (\ref{eq:weightBound5}), and the fact that $Y$ and $X_{r}-X_{l}$ have same distribution, we prove
	\begin{equation}
		\potential^{r}_{m}(\cumv + \stdv_{l}) - \potential^{r}_{m}(\cumv + \stdv_{r}) \leq 2C'k\sqrt\frac{k}{m}.
	\end{equation}
	The proof is complete by observing that $\weightv^i[t] = \sum_{l=1}^{k}[\potential^{y_{t}}_{N-i} (\cumv^{i-1}_t + \stdv_l) - \potential^{y_{t}}_{N-i} (\cumv^{i-1}_t + \stdv_{y_{t}})]$.
\end{proof}

\begin{remark}
	By summing (\ref{eq:weightBound}) over $i$, we can bound the second term of (\ref{eq:mistakeGeneralOptimal}) by $O(k^{5/2} \sqrt N) S$. Comparing this to the aforementioned bound $kNS$, Lemma \ref{lemma:weightBound} reduces the dependency on $N$, but as a tradeoff the dependency on $k$ is increased. The optimal bound for this term remains open, but in the case that the number of classes $k$ is fixed to be moderate, Lemma \ref{lemma:weightBound} provides a better bound.
\end{remark}

Corollary \ref{corollary:mistakeOptimal} is a simple consequence of plugging Lemma \ref{lemma:asymptoticError} and \ref{lemma:weightBound} to Theorem \ref{thm:mistakeOptimal}.

\subsection{Proof of lower bounds and discussion of gap}
\label{appendix:optimality}

We begin by proving Theorem \ref{thm:optimality}.
\begin{proof}
   	At time $t$, an adversary draws a label $y_{t}$ uniformly at random from $[k]$, and the weak learners independently make predictions with respect to the probability distribution $\probv_{t} \in \Delta [k]$. This can be achieved if the adversary draws $\samplex_{t} \in \reals^{N}$ where $\samplex_{t}[1], \cdots, \samplex_{t}[N]|y_{t}$'s are conditionally independent with conditional distribution of $\probv_{t}$ and $WL^{i}$ predicts $\samplex_{t}[i]$. The booster can only make a final decision by weighted majority votes of $N$ weak learners. We will manipulate $\probv_{t}$ in such a way that weak learners satisfy (\ref{eq:onlineWLC}), but the booster's performance is close to that of Online MBBM. 
    
    First we note that since $\cost_{t}[y_{t}, \predy_{t}]$ used in (\ref{eq:onlineWLC}) is bounded in $[0, 1]$, the Azuma-Hoeffding inequality implies that if a weak learner makes prediction $\predy_{t}$ according to the probability distribution $\probv_{t}$ at time $t$, then with probability $1-\delta$, we have 
    
    \begin{align}
    	\begin{split}
    	\label{eq:optimality1}
    	\sum_{t=1}^{T}w_{t}\cost_{t}[y_{t}, \predy_{t}] 
	&\leq \sum_{t=1}^{T}w_{t}\cost_{t}[y_{t}] \bullet \probv_{t}  + \sqrt {2||\weightv||^{2}_{2} \ln(\frac{1}{\delta})} \\
	&\leq \sum_{t=1}^{T}w_{t}\cost_{t}[y_{t}] \bullet \probv_{t} + \frac{\gamma ||\weightv||^{2}_{2}}{k} + \frac{k\ln(\frac{1}{\delta})}{2\gamma}\\
	&\leq \sum_{t=1}^{T}w_{t}\cost_{t}[y_{t}] \bullet \probv_{t} + \frac{\gamma ||\weightv||_{1}}{k} + \frac{k\ln(\frac{1}{\delta})}{2\gamma},
	\end{split}
    \end{align}
    where the second inequality holds by arithmetic mean and geometric mean relation and the last inequality holds due to $w_{t} \in [0, 1]$.
    
    We start from providing a lower bound on the number of weak learners. Let $\probv_{t}=\unifv^{y_{t}}_{2\gamma}$ for all $t$. This can be done by the constraint $\gamma < \frac{1}{4}$. Then the last line of (\ref{eq:optimality1}) becomes 
    \begin{equation*}
    	\sum_{t=1}^{T}w_{t}\cost_{t}[y_{t}] \bullet \unifv^{y_{t}}_{2\gamma} + \frac{\gamma ||\weightv||_{1}}{k} + \frac{k\ln(\frac{1}{\delta})}{2\gamma} 
	= \frac{1-2\gamma}{k}||\weightv||_{1} + \frac{\gamma ||\weightv||_{1}}{k} + \frac{k\ln(\frac{1}{\delta})}{2\gamma} 
	\leq \frac{1-\gamma}{k}||\weightv||_{1} + S,
    \end{equation*}
    where the first equality follows by the fact that $\cost_{t}[y_{t}, y_{t}] = 0$ and $||\cost_{t}[y_{t}]||_{1}=1$. Thus the weak learners indeed satisfy the online weak learning condition with edge $\gamma$ and excess loss $S$. Now suppose a booster imposes weights on weak learners by $\alpha^{i}$. WLOG, we may assume the weights are normalized such that $\sum_{i=1}^{N}\alpha^{i}=1$. Adopting the argument of \citet[Section 13.2.6]{schapire2012boosting}, we prove that the optimal choice of weights is $(\frac{1}{N}, \cdots, \frac{1}{N})$. Fix $t$, and let $l^{i}$ denote the prediction made by $WL^{i}$. By noting that $\prob(y_{t} = y) = \frac{1}{k}$, which is constant, we can deduce
    \begin{align*}
    	\prob(y_{t} = y | l^{1}, \cdots, l^{N}) & = \frac{\prob(l^{1}, \cdots, l^{N}|y_{t}=y) \prob(y_{t}=y)}{\prob(l^{1}, \cdots, l^{N})}\\
	&\propto \prob(l^{1}, \cdots, l^{N}|y_{t}=y) \\
	&= \prod_{i=1}^{N}p^{\ind(l^{i}=y)}q^{\ind(l^{i}\neq y)},
    \end{align*}
    where $f \propto g$ means $f(y)/g(y)$ does not depend on $y$, $p = \unifv^{y_{t}}_{2\gamma}[y_{t}] = \frac{1-2\gamma}{k}+2\gamma$, and $q = \unifv^{y_{t}}_{2\gamma}[l] = \frac{1-2\gamma}{k}$. By taking log, we get 
    \begin{align*}
    	\log \prob(y_{t} = y | l^{1}, \cdots, l^{N}) &= C + \log p \sum_{i=1}^{N} \ind(l^{i}=y) + \log q \sum_{i=1}^{N} \ind(l^{i}\neq y) \\
	&= C + N \log q + \log \frac{p}{q} \sum_{i=1}^{N} \ind(l^{i}=y) .
    \end{align*}
    Therefore, the optimal decision after observing $l^{1}, \cdots, l^{N}$ is to choose $y$ that maximizes $\sum_{i=1}^{N} \ind(l^{i}=y)$, or equivalently, to take simple majority votes. 
        
    To compute a lower bound for the error rate, we again introduce random draw framework as in the proof of Lemma \ref{lemma:asymptoticError}. WLOG,  we may assume that the true label is $1$. Let $A_{i}$ denote the event that the number $i$ beats $1$ in the majority votes. Then we have 
    \begin{equation}
    	\label{eq:optimality5}
    	\prob(\text{booster makes error}) \geq \prob(A_{2}).
    \end{equation}
    Now we need a lower bound for $\prob(A_{2})$. To do so, let $\{Y_{i}\}$ be the series of i.i.d. random variables such that $Y_{i} \in \{-1, 0, 1\}$ and  
    \begin{align*}
    	\prob(Y_{j} = 1) &= \frac{1-2\gamma}{k} + 2\gamma=:p_{1}, \\
    	 \prob(Y_{j} = -1) &=\frac{1-2\gamma}{k}=:p_{-1}.			
    \end{align*}
    Then $\prob(A_{2}) = \prob(Y < 0)$ where $Y := \sum_{i=1}^{N}Y_{i}$. 
    
    Now let $M$ be the number of $j$ such that $Y_{j} \neq 0$. By conditioning on $M$, we can write 
    \begin{equation*}
	\label{eq:optimality7}
    	\prob(Y < 0 | M=m) = \prob(B \leq \frac{m}{2}),
    \end{equation*}
    where $B \sim binom(m, \frac{p_{1}}{p_{1}+ p_{-1}})$. By Slud's inequality \cite[Theorem 2.1]{slud1977distribution}, we have 
    \begin{equation*}
    	\prob(B \leq \frac{m}{2}) \geq \prob(Z \geq \sqrt{m} \frac{p-\frac{1}{2}}{\sqrt{p(1-p)}}),
    \end{equation*}
    where $Z$ follows a standard normal distribution and $p = \frac{p_{1}}{p_{1}+p_{-1}}$. Now using tail bound on normal distribution, we get 
    \begin{align}
    	\label{eq:optimality8}
	\begin{split}
		\prob(B \leq \frac{m}{2}) &\geq \Omega(\exp(-\frac{m(p-1/2)^{2}}{p(1-p)}) )\\
		&= \Omega (\exp(- \frac{m(p_{1} - p_{-1})^{2}}{4p_{1}p_{-1}}) )\\
		&= \Omega (\exp(-\frac{m\gamma^{2}}{p_{1}p_{-1}}))\\
		&\geq \Omega (\exp(-4m k^{2} \gamma^{2}))\\
		&\geq \Omega (\exp(-4N k^{2} \gamma^{2})).
	\end{split}
    \end{align}
    We intentionally drop $\frac{1}{2}$ from the power, which makes the bound smaller. The second inequality holds because $p_{1}p_{-1} \geq \frac{(1-2\gamma)^{2}}{k^{2}} \geq \frac{1}{4k^{2}}$. Integrating w.r.t. $m$ gives 
    \begin{equation*}
    	\prob(\text{booster makes error}) \geq \prob(Y < 0) \geq \Omega (\exp(-4N k^{2} \gamma^{2})).
    \end{equation*}
    By setting this value equal to $\epsilon$, we have $N \geq \Omega(\frac{1}{k^{2}\gamma^{2}}\ln \frac{1}{\epsilon})$, which proves the first part of the theorem. 
    
    Now we turn our attention to the optimality of sample complexity. Let $T_{0}:=\frac{kS}{4\gamma}$ and define $\probv_{t} = \unifv^{y_{t}}_{0}$ for $t \leq T_{0}$ and $\probv_{t} = \unifv^{y_{t}}_{2\gamma}$ for $t > T_{0}$. Then for $T \leq T_{0}$, (\ref{eq:optimality1}) implies 
    \begin{equation}
    	\label{eq:optimality2}
    	\sum_{t=1}^{T}w_{t}\cost_{t}[y_{t}, \predy_{t}] 
	\leq \frac{1+\gamma}{k}||\weightv||_{1} + \frac{k \ln (\frac{1}{\delta})}{2\gamma} 
	\leq \frac{1-\gamma}{k}||\weightv||_{1} +S, 
    \end{equation}
    where the last inequality holds because $||\weightv||_{1} \leq T_{0} = \frac{kS}{4\gamma}$. For $ T > T_{0}$, again (\ref{eq:optimality1}) implies 
    \begin{align}
    	\begin{split}
    	\label{eq:optimality3}
    	\sum_{t=1}^{T}w_{t}\cost_{t}[y_{t}, \predy_{t}] 
	&\leq \frac{1}{k}\sum_{t=1}^{T_{0}}w_{t} + \frac{1-2\gamma}{k}\sum_{t=T_{0}+1}^{T}w_{t} 
	+ \frac{\gamma ||\weightv||_{1}}{k} + \frac{k \ln(\frac{1}{\delta})}{2\gamma} \\
    	&\leq \frac{2\gamma}{k}T_{0} + \frac{1-\gamma}{k}||\weightv||_{1} + \frac{k \ln(\frac{1}{\delta})}{2\gamma}  \\
    	&\leq \frac{1-\gamma}{k}||\weightv||_{1} + S.
    	\end{split}
    \end{align}
    (\ref{eq:optimality2}) and (\ref{eq:optimality3}) prove that the weak learners indeed satisfy (\ref{eq:onlineWLC}). Now note that combining weak learners does not provide meaningful information for $t \leq T_{0}$, and thus any online boosting algorithm has errors at least $\Omega(T_{0})$. Therefore to get the desired asymptotic error rate, the number of observations $T$ should be at least $\Omega(\frac{T_{0}}{\epsilon}) = \Omega(\frac{k}{\epsilon \gamma}S)$, which proves the second part of the theorem. 
\end{proof}

\begin{figure}[t!]
	\centering
	\begin{tabular}{cc}
		\includegraphics[width=0.4\textwidth]{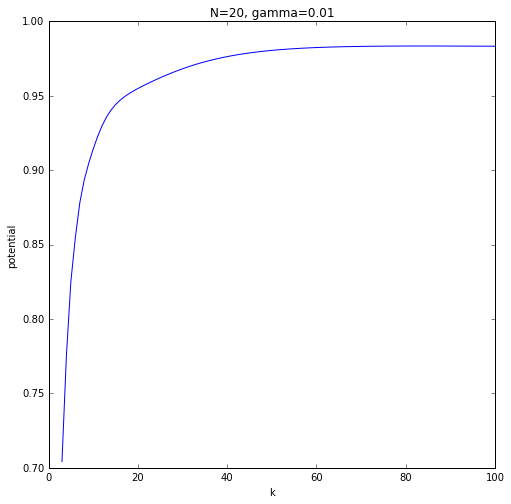} & \includegraphics[width=0.4\textwidth]{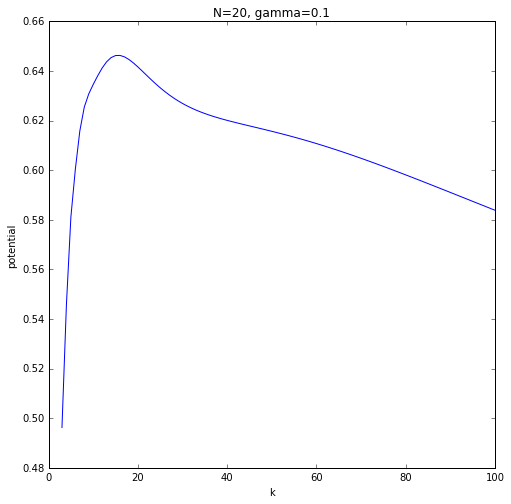}
	\end{tabular}
	\caption{Plot of $\potential^{1}_{N}(\textbf{0})$ computed with distribution $\unifv^{1}_{\gamma}$ versus the number of labels $k$. $N$ is fixed to be 20, and the edge $\gamma$ is set to be 0.01 (left) and 0.1 (right). The graph is not monotonic for larger edge. This hinders the approximation of potential functions with respect to $k$.}
	\label{fig:potential}
\end{figure}

Even though the gap for the number of weak learners between Corollary \ref{corollary:mistakeOptimal} and Theorem \ref{thm:optimality} is merely polynomial in $k$, readers might think it is counter-intuitive that $N$ is increasing in $k$ in the upper bound while decreasing in the lower bound. This phenomenon occurs due to the difficulty in approximating potential functions. Recall that Lemma \ref{lemma:asymptoticError} and Theorem \ref{thm:optimality} utilize upper and lower bound of $\potential^{1}_{N}(\textbf{0})$.

At first glance, considering that $\potential^{1}_{N}(\textbf{0})$ implies the error rate of majority votes out of $N$ independent random draws with distribution $\unifv^{1}_{\gamma}$, the potential function seems to be increasing in $k$ as the task gets harder with bigger set of options. This is the case of left panel of Figure \ref{fig:potential}. However, as it is shown in the right panel, it can also start decreasing in $k$ when $\gamma$ is larger. This can happen because the probability that a wrong label is drawn vanishes as $k$ grows while the probability that the correct label is drawn remains bigger than $\gamma$. In this regard, even though the number of wrong labels gets larger, the error rate actually decreases as $\unifv^{1}_{\gamma}[1]$ dominates other probabilities. 

After acknowledging that $\potential^{1}_{N}(\textbf{0})$ might not be a monotonic function of $k$, the linear upper bound (\ref{eq:asymptoticError}) turns out to be quite naive, and this is the main reason for the conflicting dependence on $k$ in upper bound and lower bound for $N$. As the relation among $k$, $N$, and $\gamma$ in $\potential^{1}_{N}(\textbf{0})$ is quite intricate, the issue of deriving better approximation of potential functions remains open.

\section{Proof of Theorem \ref{thm:mistakeAdaptive}}
\label{appendix:mistakeAdaptive}
We first introduce a lemma that will be used in the proof. 
\begin{lemma}
\label{lemma:1}
Suppose $A, B \geq 0$, $B-A = \gamma \in [-1, 1]$, and $A+B \leq 1$. Then we have 
\begin{equation*}
\min_{\alpha \in [-2, 2]} A(e^{\alpha}-1) + B(e^{-\alpha}-1) \leq -\frac{\gamma^{2}}{2}.
\end{equation*}
\end{lemma}

\begin{proof}
We divide into three cases with respect to the range of $\frac{B}{A}$. 

First suppose $e^{-4} \leq \frac{B}{A} \leq e^4$. In this case, the minimum is attained at $\alpha = \frac{1}{2}\log \frac{B}{A}$, and the minimum becomes 
\begin{align*}
    -(A+B) + 2 \sqrt{AB} &= - (\sqrt{A} - \sqrt{B})^2 \\
    &= -(\frac{A-B}{\sqrt A + \sqrt B})^2 \\
    &= - \frac{\gamma^2}{(\sqrt A + \sqrt B)^2} \\ 
    &\leq - \frac{\gamma^2}{2(A+B)} \leq -\frac{\gamma^2}{2}.
\end{align*}

Now suppose $\frac{B}{A} > e^4 > 51$. From $B-A = \gamma$, we have $\gamma > 50A\geq 0$. Choosing $\alpha = \log 6$, we get the minimum is bounded above by 
\begin{align*}
    5A - \frac{5}{6}B &= \frac{25}{6}A - \frac{5}{6}\gamma \\
    &< \frac{25}{6} \frac{\gamma}{50} - \frac{5}{6}\gamma \\
    &= -\frac{3}{4} \gamma < -\frac{\gamma^2}{2}.
\end{align*}
The last inequality hold due to $\gamma \leq 1$. 

Finally suppose $\frac{A}{B} > e^4 > 51$. From $B-A = \gamma$, we have $-\gamma > 50B \geq 0$. Choosing $\alpha = -\log 6$, we get the minimum is bounded above by 
\begin{align*}
    -\frac{5}{6}A + 5B &= \frac{25}{6}B + \frac{5}{6}\gamma \\
    &< -\frac{25}{6}\frac{\gamma}{50} + \frac{5}{6}\gamma \\
    & = \frac{3}{4}\gamma < -\frac{\gamma^2}{2}.
\end{align*}
The last inequality hold due to $\gamma \geq -1$. This completes the proof. 
\end{proof}

Now we provide a proof of Theorem \ref{thm:mistakeAdaptive}.

\begin{proof}
Let $M_{i}$ denote the number of mistakes made by expert $i$: $M_{i} = \sum_{t} \ind (y_{t} \neq \predy^{i}_{t})$. We also let $M_{0} = T$ for the ease of presentation. As Adaboost.OLM is using the Hedge algorithm among $N$ experts, the Azuma-Hoeffding inequality and a standard analysis (cf. \citet[Corollary 2.3]{cesa2006prediction}) provide with probability $1-\delta$, 
\begin{equation}
\label{eq:main1}
\sum_{t} \ind(y_{t} \neq \predy_{t}) \leq 2 \min_{i} M_{i} + 2 \log N + \tilde O (\sqrt T),
\end{equation}
where $\tilde O$ notation suppresses dependence on $\log \frac{1}{\delta}$.

Now suppose the expert $i-1$ makes a mistake at iteration $t$. That is to say, in a conservative way, $\cumv^{i-1}_{t}[y_{t}] \leq \cumv^{i-1}_{t}[l]$ for some $l \neq y_{t}$. This implies that among $k-1$ terms in the summation of $-\cost^{i}_{t}[y_{t}, y_{t}]$ in (\ref{eq:costLogistic}), at least one term is not less than $\frac{1}{2}$. Thus we can say $-\cost^{i}_{t}[y_{t}, y_{t}] \geq \frac{1}{2}$ if the expert $i-1$ makes a mistake at $\samplex_{t}$. This leads to the inequality:
\begin{equation}
\label{eq:main2}
-\sum_{t}\cost^{i}_{t}[y_{t}, y_{t}] \geq \frac{M_{i-1}}{2}.
\end{equation}
Note that by definition of $M_{0}$ and $\cost^{1}_{t}$, the above inequality holds for $i=1$ as well.  For ease of notation, let us write $w^{i} :=-\sum_{t}\cost^{i}_{t}[y_{t}, y_{t}]$. 

Now let $\Delta_{i}$ denote the difference of the cumulative logistic loss between two consecutive experts: 
\begin{equation*}
\Delta_{i} = \sum_{t} L^{y_{t}}(\cumv^{i}_{t}) - L^{y_{t}}(\cumv^{i-1}_{t}) = \sum_{t}L^{y_{t}}(\cumv^{i-1}_{t} + \alpha^{i}_{t}\stdv_{l^{i}_{t}}) - L^{y_{t}}(\cumv^{i-1}_{t}). 
\end{equation*}
Then Online Gradient Descent algorithm provides 
\begin{equation}
\label{eq:main3}
\Delta_{i} \leq \min_{\alpha \in [-2, 2]} \sum_{t} [L^{y_{t}}(\cumv^{i-1}_{t} + \alpha \stdv_{l^{i}_{t}}) - L^{y_{t}}(\cumv^{i-1}_{t})]  + 4\sqrt 2 (k-1) \sqrt T. 
\end{equation}

By simple algebra, we can check 
\begin{equation*}
\log(1+e^{s+\alpha}) - \log(1+e^s) = \log(1 + \frac{e^{\alpha}-1}{1+ e^{-s}}) \leq \frac{1}{1 + e^{-s}} (e^\alpha - 1).
\end{equation*}
From this, we can deduce that
\begin{align*}
L^{y_t} (\cumv^{i-1}_t + \alpha \stdv_{l^{i}_t}) - L^{y_t}(\cumv^{i-1}_t) \leq 
\begin{cases}
\cost^{i}_t [y_t, l^{i}_t] (e^\alpha -1)&, \text{if } l^{i}_t \neq y_t \\
\cost^{i}_t [y_t, l^{i}_t] (-e^{-\alpha} +1)&, \text{if } l^{i}_t = y_t \\
\end{cases}.
\end{align*}
Summing over $t$, we have 
\begin{equation*}
\sum_{t}L^{y_t} (\cumv^{i-1}_t + \alpha \stdv_{l^{i}_t}) - L^{y_t}(\cumv^{i-1}_t)  \leq w^{i} (A (e^\alpha -1) +B (e^{-\alpha} -1)),
\end{equation*}
where 
\begin{equation*}
 A = \sum_{l_t \neq y_t} \cost_t [y_t, l_t] / w^{i},~ B = -\sum_{l_t = y_t} \cost_t [y_t, l_t] / w^{i}.
\end{equation*}
Note that $A$ and $B$ are non-negative and $B-A = \gamma_{i} \in [-1, 1], ~ A+B \leq 1$. Lemma \ref{lemma:1} provides
\begin{equation}
\label{eq:main5}
\min_{\alpha \in [-2, 2]} \sum_{t} [L^{y_{t}}(\cumv^{i-1}_{t} + \alpha \stdv_{l^{i}_{t}}) - L^{y_{t}}(\cumv^{i-1}_{t})] \leq -\frac{{\gamma_{i}}^{2}}{2} w^{i}.
\end{equation}

Combining (\ref{eq:main2}), (\ref{eq:main3}), and (\ref{eq:main5}), we have

\begin{equation*}
\Delta_{i} \leq -\frac{{\gamma_{i}}^{2}}{4} M_{i-1} + 4\sqrt 2 (k-1) \sqrt T.
\end{equation*}

Summing over $i$, we get by telescoping rule
\begin{align*}
\sum_{t} L^{y_{t}}(\cumv^{N}_{t}) - \sum_{t} L^{y_{t}}(\zerov) &\leq -\frac{1}{4}\sum_{i} \gamma_{i}^{2} M_{i-1} + 4\sqrt 2 (k-1) N\sqrt T \\
&\leq -\frac{1}{4}\sum_{i} \gamma_{i}^{2} \min_{i} M_{i} + 4\sqrt 2 (k-1)N \sqrt T.
\end{align*}
Note that $L^{y_{t}}(\zerov) = (k-1) \log 2$ and $L^{y_{t}}(\cumv^{N}_{t}) \geq 0$. Therefore we have
\begin{equation*}
\min_{i}M_{i} \leq  \frac{4 (k-1) \log 2}{\sum_{i}\gamma_{i}^{2}}T + \frac{16\sqrt 2 (k-1)N}{\sum_{i}\gamma_{i}^{2}} \sqrt T.
\end{equation*}
Plugging this in (\ref{eq:main1}), we get with probability $1-\delta$,
\begin{align*}
\sum_{t} \ind(y_{t} \neq \predy_{t}) &\leq \frac{8 (k-1) \log 2}{\sum_{i}\gamma_{i}^{2}}T + \tilde O (\frac{kN \sqrt T}{\sum_{i}\gamma_{i}^{2}} + \log N) \\
&\leq \frac{8 (k-1)}{\sum_{i}\gamma_{i}^{2}}T + \tilde O (\frac{kN^{2}}{\sum_{i}\gamma_{i}^{2}}),
\end{align*}
where the last inequality holds from AM-GM inequality: $cN\sqrt T \leq \frac{c^{2}N^{2} + T}{2}$. 

\end{proof}

\section{Adaptive algorithms with different surrogate losses}
\label{appendix:differentLosses}
In this section, we present similar adaptive boosting algorithms with Adaboost.OLM but with two different surrogate losses: exponential loss and square hinge loss. We keep the main structure, but the unique properties of each loss result in little difference in details. 

\subsection{Exponential loss}
As discussed in Section \ref{section:surrogate}, exponential loss is useful in batch setting because it provides a closed form for the potential function. We will use following multiclass version of exponential loss:
\begin{equation}
\label{eq:expLoss}
L^{r}(\cumv) := \sum_{l\neq r} \exp(\cumv[l] - \cumv[r]).
\end{equation}

From this, we can compute the cost matrix and ${f^{i}_{t}}'$ for the online gradient descent as below:

\begin{align}
\label{eq:expCost}
\cost^{i}_t[r,l] = 
    \begin{cases} 
     \exp(\cumv^{i-1}_t [l] - \cumv^{i-1}_t[r])&, \text{if } l\neq r\\
    -\sum_{j\neq r}\exp(\cumv^{i-1}_t [j] - \cumv^{i-1}_t[r])&, \text{if } l = r
    \end{cases}
\end{align}

\begin{align}
\label{eq:expGradient}
{f^i_t}'(\alpha) = 
    \begin{cases} 
    \exp(\cumv^{i-1}_t[l^i_t]+\alpha - \cumv^{i-1}_t [y_{t}])&, \text{if } l^i_t\neq y_t\\
    -\sum_{j\neq y_t}\exp(\cumv^{i-1}_t[j]-\alpha - \cumv^{i-1}_t [y_{t}])&, \text{if } l^i_t = y_t.
    \end{cases}
\end{align}
With this gradient, if we set the learning rate $\eta^{i}_{t} = \frac{2\sqrt 2}{(k-1)\sqrt t} e^{-i}$, a standard analysis provides $R^{i}(T) \leq 4\sqrt 2 (k-1) e^{i} \sqrt T$. Note that with exponential loss, we have different learning rate for each weak learner. We keep the algorithm same as Algorithm \ref{alg:AdaboostOLM}, but with different cost matrix and learning rate. Now we state the theorem for the mistake bound. 

\begin{theorem}{\bf{(Mistake bound with exponential loss)}}
\label{thm:exp}
For any $T$ and $N$, the number of mistakes made by Algorithm \ref{alg:AdaboostOLM} with above cost matrix and learning rate satisfies the following inequality with high probability: 
\begin{equation*}
\sum_{t} \ind(y_{t} \neq \predy_{t}) \leq \frac{4k}{\sum_{i}\gamma_{i}^{2}}T + \tilde O (\frac{ke^{2N}}{\sum_{i}\gamma_{i}^{2}}).
\end{equation*}
\end{theorem}
\begin{proof}
The proof is almost identical to that of Theorem \ref{thm:mistakeAdaptive}, and we only state the different steps. With cost matrix defined in (\ref{eq:expCost}), we can show 
\begin{equation*}
-\sum_{t} \cost^{i}_{t}[y_{t}, y_{t}] \geq M_{i-1}.
\end{equation*}
Furthermore, we have following identity (which was inequality in the original proof):
\begin{align*}
L^{y_t} (\cumv^{i-1}_t + \alpha \stdv_{l^{i}_t}) - L^{y_t}(\cumv^{i-1}_t) =
\begin{cases}
\cost^{i}_t [y_t, l^{i}_t] (e^\alpha -1)&, \text{if } l^{i}_t \neq y_t \\
\cost^{i}_t [y_t, l^{i}_t] (-e^{-\alpha} +1)&, \text{if } l^{i}_t = y_t \\
\end{cases}.
\end{align*}
This leads to
\begin{equation*}
\Delta_{i} \leq -\frac{{\gamma_{i}}^{2}}{2} M_{i-1} + 4\sqrt 2 (k-1) e^{i} \sqrt T.
\end{equation*}
Summing over $i$, we get 
\begin{align*}
\frac{\sum_{i} \gamma_{i}^{2}}{2} \min_{i}M_{i} &\leq (k-1)T + 4\sqrt 2 (k-1) e \frac{e^{N}-1}{e-1} \sqrt T \\
&\leq (k-1)T + 9k e^{N}\sqrt T.
\end{align*}
Plugging this in (\ref{eq:main1}), we get with high probability,
\begin{align*}
\sum_{t} \ind(y_{t} \neq \predy_{t}) &\leq \frac{4 (k-1)}{\sum_{i}\gamma_{i}^{2}}T + \tilde O (\frac{ke^{N} \sqrt T}{\sum_{i}\gamma_{i}^{2}} + \log N) \\
&\leq \frac{4 k}{\sum_{i}\gamma_{i}^{2}}T + \tilde O (\frac{ke^{2N}}{\sum_{i}\gamma_{i}^{2}}),
\end{align*}
which completes the proof. We also used AM-GM inequality for the last step.
\end{proof}

Comparing to Theorem \ref{thm:mistakeAdaptive}, we get a better coefficient for the first term, which is asymptotic error rate, but the exponential function in the second term makes the bound significantly loose. The exponential term comes from the larger variability of $f^{i}_{t}$ associated with exponential loss. It should also be noted that the empirical edge $\gamma_{i}$ is measured with different cost matrices, and thus direct comparison is not fair. In fact, as discussed in Section \ref{section:surrogate}, $\gamma_{i}$ is closer to $0$ with exponential loss than with logistic loss due to larger variation in weights, which is another huge advantage of logistic loss. 

\subsection{Square hinge loss}
Another popular surrogate loss is square hinge loss. We begin the section by introducing multiclass version of it:
\begin{equation}
\label{eq:squareLoss}
L^{r}(\cumv) := \frac{1}{2}\sum_{l\neq r} (\cumv[l] - \cumv[r]+1)^{2}_{+},
\end{equation}
where $f_{+} := \max\{0, f\}$. From this, we can compute the cost matrix and ${f^{i}_{t}}'$ for the online gradient descent as below:

\begin{align}
\label{eq:squareCost}
\cost^{i}_t[r,l] = 
    \begin{cases} 
     (\cumv^{i-1}_t [l] - \cumv^{i-1}_t[r]+1)_{+}&, \text{if } l\neq r\\
    -\sum_{j\neq r}(\cumv^{i-1}_t [j] - \cumv^{i-1}_t[r]+1)_{+}&, \text{if } l = r
    \end{cases}
\end{align}

\begin{align}
\label{eq:squareGradient}
{f^i_t}'(\alpha) = 
    \begin{cases} 
    (\cumv^{i-1}_t [l^{i}_{t}] + \alpha - \cumv^{i-1}_t[y_{t}]+1)_{+}&, \text{if } l^i_t\neq y_t\\
    -\sum_{j\neq y_t}(\cumv^{i-1}_t [j] - \alpha - \cumv^{i-1}_t[y_{t}]+1)_{+}&, \text{if } l^i_t = y_t.
    \end{cases}
\end{align}

With square hinge loss, we do not use Lemma \ref{lemma:1} in the proof of mistake bound, and thus the feasible set $F$ can be narrower. In fact, we will set $F = [-c, c]$, where the parameter $c$ will be optimized later. With this $F$, we have $|{f^i_t}'(\alpha)| \leq (k-1) + ci\leq (k-1) + cN$, and the standard analysis of online gradient descent algorithm with learning rate $\eta_{t} = \frac{\sqrt 2 c}{((k-1)+cN)\sqrt t}$ provides that $R^{i}(T) \leq 2\sqrt 2 (k-1 + cN)\sqrt T$. Now we are ready to prove the mistake bound. 

\begin{theorem}{\bf{(Mistake bound with square hinge loss)}} 
\label{thm:square}
For any $T$ and $N$, with the choice of $c = \frac{1}{\sqrt N}$, the number of mistakes made by Algorithm \ref{alg:AdaboostOLM} with above cost matrix and learning rate satisfies the following inequality with high probability:
\begin{equation*}
\sum_{t} \ind(y_{t} \neq \predy_{t}) \leq \frac{2k \sqrt N}{\sum_{i}|\gamma_{i}|}T + \tilde O (\frac{(k^{2} + N)N \sqrt N}{\sum_{i}|\gamma_{i}|}).
\end{equation*}
\end{theorem}

\begin{proof}
With cost matrix defined in (\ref{eq:squareCost}), we can show
\begin{equation*}
-\sum_{t} \cost^{i}_{t}[y_{t}, y_{t}] \geq M_{i-1}.
\end{equation*}

We can also check that 
\begin{equation*}
\frac{1}{2}[(s+\alpha)^{2}_{+} - s^{2}_{+}] \leq s_{+} \alpha + \frac{\alpha^{2}}{2},
\end{equation*}
by splitting the cases with the sign of each term. Using this, we can deduce that 

\begin{equation*}
L^{y_t} (\cumv^{i-1}_t + \alpha \stdv_{l^{i}_t}) - L^{y_t}(\cumv^{i-1}_t) \leq \cost^{i}_{t}[y_{t}, l^{i}_{t}] \alpha + \frac{(k-1)\alpha^{2}}{2}.
\end{equation*}

Summing over $t$ gives 
\begin{equation*}
\sum_{t}L^{y_t} (\cumv^{i-1}_t + \alpha \stdv_{l^{i}_t}) - L^{y_t}(\cumv^{i-1}_t) 
\leq \sum_{t} \cost^{i}_{t}[y_{t}, y_{t}] \gamma_{i} \alpha + \frac{(k-1)\alpha^{2}}{2}T.
\end{equation*}

The RHS is a quadratic in $\alpha$, and the minimizer is $\alpha^{*} = -\frac{\sum_{t} \cost^{i}_{t}[y_{t}, y_{t}] \gamma_{i}}{(k-1)T}$. Since the magnitude of $\cost^{i}_{t}[y_{t}, y_{t}]$ grows as a function of $c$, there is no guarantee that this minimizer lies in the feasible set $F = [-c, c]$. Instead, we will bound the minimum by plugging in $\alpha = \pm c$:
\begin{align*}
\min_{\alpha \in [-c, c]}\sum_{t}L^{y_t} (\cumv^{i-1}_t + \alpha \stdv_{l^{i}_t}) - L^{y_t}(\cumv^{i-1}_t) 
&\leq \frac{(k-1)c^{2}}{2} T + c|\gamma_{i}|\sum_{t} \cost^{i}_{t}[y_{t}, y_{t}] \\
&\leq \frac{(k-1)c^{2}}{2} T - c|\gamma_{i}|M_{i-1}.
\end{align*}

From this, we get 
\begin{equation*}
\Delta_{i} \leq -c|\gamma_{i}|M_{i-1} + \frac{(k-1)c^{2}}{2} T + 2\sqrt 2 (k-1 + cN)\sqrt T.
\end{equation*}

Summing over $i$, we get
\begin{align*}
c \sum_{i}|\gamma_{i}| \min_{i}M_{i} \leq \frac{k-1}{2}T + \frac{(k-1)c^{2}N}{2} T + 2\sqrt 2 (k-1 + cN)N\sqrt T.
\end{align*}
By rearranging terms, we conclude 
\begin{equation*}
\min_{i}M_{i} \leq \frac{(k-1)}{2 \sum_{i}|\gamma_{i}|} (\frac{1}{c} + cN)T + \frac{2\sqrt 2 (k-1 + cN)N}{\sum_{i}|\gamma_{i}|} \sqrt T. 
\end{equation*}
It is the first term from the RHS that provides an optimal choice of $c = \frac{1}{\sqrt N}$, and this value gives 
\begin{equation*}
\min_{i} M_{i} \leq \frac{(k-1) \sqrt N}{\sum_{i}|\gamma_{i}|}T + \frac{2\sqrt 2 (k-1 + \sqrt N)N}{\sum_{i}|\gamma_{i}|} \sqrt T.
\end{equation*}
Plugging this in (\ref{eq:main1}), we get with high probability,
\begin{align*}
\sum_{t} \ind(y_{t}\neq \predy_{t}) 
&\leq \frac{2(k-1) \sqrt N}{\sum_{i}|\gamma_{i}|}T + \tilde O (\frac{(k + \sqrt N)N}{\sum_{i}|\gamma_{i}|}\sqrt T + \log N) \\
&\leq \frac{2k \sqrt N}{\sum_{i}|\gamma_{i}|}T + \tilde O (\frac{(k^{2} + N)N \sqrt N}{\sum_{i}|\gamma_{i}|}),
\end{align*}
which completes the proof. We also used AM-GM inequality for the last step. 
\end{proof}

By Cauchy-Schwartz inequality, we have $N \sum_{i}\gamma_{i}^{2} \geq (\sum_{i}|\gamma_{i}|)^{2}$. From this, we can deduce $(\frac{\sqrt N}{\sum_{i}|\gamma_{i}|})^{2} \geq \frac{1}{\sum_{i}\gamma_{i}^{2}}$.
If LHS is greater than $1$, then the bound in Theorem \ref{thm:square} is meaningless. Otherwise, we have 
\begin{equation*}
\frac{\sqrt N}{\sum_{i}|\gamma_{i}|} 
\geq (\frac{\sqrt N}{\sum_{i}|\gamma_{i}|})^{2} 
\geq \frac{1}{\sum_{i}\gamma_{i}^{2}},
\end{equation*}
which validates that the bound with logistic loss is tighter. Furthermore, square hinge loss also produces more variable weights over instances, which results in worse empirical edges. 

\section{Detailed description of experiment}
\label{appendix:experiment}
Testing was performed on a variety of data sets described in Table \ref{tab:dataDetails}. All are from the UCI data repository (\citet{UCI1998}, \citet{mice2015}, \citet{ugulino2012wearable}) with a few adjustments made to deal with missing data and high dimensionality. These changes are noted in the table below. Many of the data sets are the same as used in the \citet{oza2005online}, with the addition of a few sets with larger numbers of data points and predictors. We report the average performance on both the entire data set and on the final 20\% of the data set. The two accuracy measures help understand both the ``burn in period'', or how quickly the algorithm improves as observations are recorded, and the ``accuracy plateau'', or how well the algorithm can perform given sufficient data. Different applications may emphasize each of these two algorithmic characteristics, so we choose to provide both to the reader. We also report average run times. All computations were carried out on a Nehalem architecture 10-core 2.27 GHz Intel Xeon E7-4860 processors with 25 GB RAM per core. For all but the last two data sets, results are averaged over 27 reordering of the data. Due to computational constraints, Movement was run just nine times and ISOLET just once. 

\begin{table}[h]
    \caption{Data set details}
    \label{tab:dataDetails}
    \centering
    \begin{tabular}{lccc}
   	\toprule
        Data sets & Number of data points & Number of predictors & Number of classes \\
        \midrule
        Balance &  625 & 4 & 3 \\ 
        Mice & 1080 & $82^\star$ & 8 \\
        Cars & 1728 & 6 & 4 \\
        Mushroom & 8124 & 22 & 2 \\
        Nursery & 12960 & 8 & 4 \\
        ISOLET & 7797 & $50^{\star \star}$ & 26 \\
        Movement & $165631^{\star \star \star}$ & $12^{\star \star \star}$ & 5 \\
        \bottomrule
        & & &\\
        \multicolumn{4}{l}{${}^\star$ Missing data was replaced with 0.}\\
        \multicolumn{4}{l}{${}^{\star \star}$ The original 617 predictors were projected onto their first 50 principal components,}\\
        \multicolumn{4}{l}{which contained 80\% of the variation.}\\
        \multicolumn{4}{l}{${}^{\star \star \star}$ User information was removed, leaving only sensor position predictors. Single data }\\
        \multicolumn{4}{l}{point with missing value removed. }
    \end{tabular}

\end{table}

In all the experiments we used Very Fast Decision Trees (VFDT) from \citet{domingos2000mining} as weak learners. VFDT has several tuning parameters which relate to the frequency with which the tree splits. In all methods we assigned these randomly for each tree. Specifically for our implementation the tuning parameter  \texttt{grace\_period} was chosen randomly between 5 and 20 and the tuning parameters \texttt{split\_confidence} and \texttt{hoeffding\_tie\_threshold} randomly between 0.01 and 0.9. It is likely that this procedure would produce trees which do not perform well on specific data sets. In practice for the Adaboost.OLM it is possible to restart poorly performing trees using parameters similar to better performing trees in an automated and online (although ad hoc) fashion using the $\alpha_t^i$, and this tends to produce superior performance (as well as allow adaptivity to changes in the data distribution). However for these experiments, we did not take advantage of this to better examine the benefits of just the cost matrix framework.

Several algorithms were tested using the above specifications, but with slightly different conditions. The first three are directly comparable since they all use the same weak learners and do not require knowledge of the edge of the weak learners. DT is the best result from running 100 VFDT independently. The best was chosen after seeing the performance on the entire data set and final 20\% respectively. However the time reported was the average time for running all 100 VFDT. This was done to better see the additional cost of running the boosting framework on top of the training of the raw weak learners. OLB is an implementation of the Online Boosting algorithm in \citet[Figure 2]{oza2005online} with 100 VFDT. AdaOLM stands for Adaboost.OLM, again with 100 VFDT.

The next five algorithms (MB) tested were all variants of the OnlineMBBM but with different edge $\gamma$ values. In practice this value is never known ahead of time, but we want to explore how different edges affect the performance of the algorithm. For the ease of computation, instead of exactly finding the value of (\ref{eq:potentialCompute}), we estimated the potential functions by Monte Carlo (MC) simulations. 

The final two algorithms are slightly different implementations of the One VS All (OvA) ensemble method. In this framework multiple binary classifiers are used to solve a multiclass problem by viewing different classes as the positive class, and all others as the negative class. They then predict whether a data point is their positive class or not, and the results are used together to make a final classification. Both use VFDT as their weak learners, but with $100 \times k$ binary trees. The first method (OvA) uses $k$ versions of Adaboost.OL, each viewing one of the classes as the positive class. Recall that Adaboost.OLM in the binary setting is just Adaboost.OL by \citet{beygelzimer2015optimal}. The second (AdaOVA) produces 100 weak multiclass classifiers by grouping a $k$ binary classifiers, one for each class, and then uses Adaboost.OLM to get the final learner, treating the 100 single tree OvA's as its weak learners. In the table below we have partitioned the methods in terms of the number of weak learners since, while they all tackle the same problem, algorithms within each partition are more directly comparable since they use the same weak learners.

\begin{table}[t!]
    \caption{Comparison of algorithms on final 20\% of data set}
    \label{tab:compareMethods8020}
    \centering
    \setlength\tabcolsep{2pt}
    \begin{tabular}{lcccccccccccc}
    	\toprule
	&&\multicolumn{8}{c}{$100$ multiclass trees}&&\multicolumn{2}{c}{$100 k$ binary trees}\\
	\cmidrule{3-10}
	\cmidrule{12-13}
        Data sets && DT & OLB & AdaOLM  & MB .3 & MB .1& MB .05 & MB .01 & MB .001 && OvA & AdaOVA\\
        \midrule
        Balance &&  0.768 & 0.772 & 0.754 & 0.788 & 0.821 & 0.819 & 0.805 & 0.752 && 0.786 & 0.795 \\ 
        Mice && 0.608 & 0.399 & 0.561 & 0.572 & 0.695 & 0.663 & 0.502 & 0.467 && 0.742  & 0.667\\
        Cars && 0.924 & 0.914 & 0.930 & 0.914 & 0.885 & 0.870 & 0.836 & 0.830 && 0.946 & 0.919\\
        Mushroom && 0.999 & 1.000 & 1.000 & 0.997 & 1.000 & 1.000 & 0.999 & 0.998 && 1.000 & 1.000\\
        Nursery && 0.953 & 0.941 & 0.966 & 0.965 & 0.969 & 0.964 & 0.948 & 0.940 && 0.974 & 0.965 \\
        \midrule
        ISOLET && 0.515 & 0.149 & 0.521 & 0.453 & 0.626 & 0.635 & 0.226 & 0.165 && 0.579 & 0.570 \\
        Movement && 0.915 & 0.870 & 0.962 & 0.975 & 0.987 & 0.988 & 0.984 & 0.981 && 0.947 & 0.970  \\
        \bottomrule
    \end{tabular}
\end{table}

\begin{table}[t!]
    \caption{Comparison of algorithms on full data set}
    \label{tab:compareMethodsAll}
    \centering
    \setlength\tabcolsep{2pt}
    \begin{tabular}{lcccccccccccc}
    	\toprule
        &&\multicolumn{8}{c}{$100$ multiclass trees}&&\multicolumn{2}{c}{$100 k$ binary trees}\\
	\cmidrule{3-10}
	\cmidrule{12-13}
        Data sets && DT & OLB & AdaOLM  & MB .3 & MB .1& MB .05 & MB .01 & MB .001 && OvA & AdaOVA\\
        \midrule
        Balance &&  0.734 & 0.747 & 0.698 & 0.751 & 0.769 & 0.759 & 0.736 & 0.677 && 0.724 & 0.730\\ 
        Mice && 0.499 & 0.315 & 0.454 & 0.457 & 0.507 & 0.449 & 0.356 & 0.343 && 0.586 & 0.530\\
        Cars && 0.848 & 0.839 & 0.865 & 0.842 & 0.829 & 0.814 & 0.767 & 0.762 && 0.881 & 0.853\\
        Mushroom && 0.996 & 0.997 & 0.995 & 0.991 & 0.995 & 0.994 & 0.993 & 0.992 && 0.996 & 0.995 \\
        Nursery && 0.921 & 0.909 & 0.928 & 0.932 & 0.936 & 0.932 & 0.918 & 0.912 && 0.939 & 0.932 \\
        \midrule
        ISOLET && 0.395 & 0.104 & 0.456 & 0.333 & 0.486 & 0.461 & 0.152 & 0.111 && 0.507 & 0.472 \\
        Movement && 0.898 & 0.864 & 0.942 & 0.954 & 0.972 & 0.973 & 0.959 & 0.957 && 0.927 & 0.952 \\
        \bottomrule
    \end{tabular}
\end{table}

\begin{table}[t!]
    \caption{Comparison of algorithms total run time in seconds}
    \label{tab:compareMethodsTimes}
    \centering
    \setlength\tabcolsep{2pt}
    \begin{tabular}{lcccccccccccc}
    	\toprule
        &&\multicolumn{8}{c}{$100$ multiclass trees}&&\multicolumn{2}{c}{$100 k$ binary trees}\\
	\cmidrule{3-10}
	\cmidrule{12-13}
        Data sets && DT & OLB & AdaOLM  & MB .3 & MB .1& MB .05 & MB .01 & MB .001 && OvA & AdaOVA\\
        \midrule
        Balance &&  8 & 19 & 20 & 26 & 42 & 47 & 50 & 51 && 66 & 43 \\ 
        Mice && 105 & 263 & 416 & 783 & 2173 & 3539 & 3579 & 3310 && 3092 & 3013\\
        Cars && 39 & 27 & 59 & 56 & 105 & 146 & 165 & 152 && 195 & 143 \\
        Mushroom && 241 & 169 & 355& 318 & 325 & 326 & 324 & 321  && 718 & 519 \\
        Nursery && 526 & 302 & 735 & 840 & 1510 & 2028 & 2181 & 1984 && 2995 & 1732 \\
        \midrule
        ISOLET && 470 & 1497 & 2422 & 18732 & 38907 & 64707 & 62492 & 50700 && 37300 & 33328 \\
        Movement && 1960 & 3437 & 5072 & 13018 & 17608 & 18676 & 16739 & 16023 && 30080 & 21389 \\
        \bottomrule
    \end{tabular}
\end{table}

\subsection{Analysis}

It is worth beginning by noting the strength of the VFDT without any boosting framework. While the results above are for the best performing tree in hindsight, which is not a valid strategy in practice, in many applications it would be possible to collect some data beforehand activating the system, and use that to pick tuning parameters. It is also worth noting that many of the weaknesses of the above methods, such as their poor scaling with the number of predictors, are also inherited from the VFDT. Nonetheless in almost all cases Adaboost.OLM algorithm outperforms both the best tree and the preexisting Online Boosting algorithm (and is often comparable to the OnlineMBBM algorithms), as well as provide theoretical guarantees. In particular these performance gains seem to be greater on the final 20\% of the data and in data sets with larger number of data points $n$, leading us to believe that Adaboost.OLM has a longer burn in period, but higher accuracy plateau. This performance does come at additional computational cost, but this cost is relatively mild, especially compared to the costs of OnlineMBBM and the OvA methods. 

The OnlineMBBM methods use additional assumptions about the power of their weak learners, and are able to leverage that additional information to produce more accurate, with one of these algorithms often achieving the highest accuracy on each data set. However they can be sensitive to the choice of $\gamma$, with the  worst choice of $\gamma$ often underperforming both pure trees and Adaboost.OLM, and with no single $\gamma$ value always producing the best result. These methods are also much slower than Adaboost.OLM, likely due to computational burden in estimating the potential functions. 

Finally our two OvA algorithms tend to perform very well, often beating the other adaptive methods. However this performance is likely due to the use of many times more weak learners than the other adaptive methods used, which results in high computational cost. Again we see that as $n$ increases the implementation of OvA using our cost matrix framework performs better compared to the vanilla implementation, reinforcing our belief that the cost matrix framework requires more data to come online but has a higher accuracy plateau.  
\end{appendices}

\end{document}